\documentclass[runningheads]{llncs}

\usepackage{eccv}

\usepackage{eccvabbrv}

\usepackage{graphicx}
\usepackage{booktabs}
\usepackage{multirow}
\usepackage{wrapfig}
\usepackage{placeins}

\usepackage[accsupp]{axessibility}  

\usepackage[breaklinks,colorlinks,citecolor=eccvblue]{hyperref}

\usepackage{orcidlink}

\usepackage{algorithm}
\usepackage{algpseudocode}

\usepackage{multibib}
\newcites{app}{Supplementary References} 

\definecolor{myGreen}{rgb}{0.37, 0.83, 0.35}

\definecolor{myRed}{rgb}{1, 0, 0}

\definecolor{myBlue}{rgb}{0, 0.6, 1}

\newcommand{\paragraphcustom}[1]{\vspace{5pt}\noindent\textbf{#1}}

\begin{document}









%


\title{Persistent Robot World Models: Stabilizing Multi-Step Rollouts via Reinforcement Learning}








\titlerunning{PersistWorld: Persistent Robot World Models via RL}

\author{Jai Bardhan~\orcidlink{0000-0001-9348-3459} \and
Patrik Drozdik~\orcidlink{0009-0007-0689-2918} \and
Josef Sivic~\orcidlink{0000-0002-2554-5301} \and
Vladimir Petrik~\orcidlink{0000-0001-9450-4987}}

\authorrunning{J.~Bardhan et al.}

\institute{Czech Institute of Informatics, Robotics and Cybernetics, Czech Technical University in Prague \\
\email{\{first.last\}@cvut.cz}\\
\url{https://www.jaibardhan.com/persistworld}}

\maketitle

\begin{abstract}
Action-conditioned robot world models generate future video frames of the manipulated scene given a robot action sequence, offering a promising alternative for simulating tasks that are difficult to model with traditional physics engines. However, these models are optimized for short-term prediction and break down when deployed autoregressively: each predicted clip feeds back as context for the next, causing errors to compound and visual quality to rapidly degrade. We address this through the following contributions. 
First, we introduce a reinforcement learning (RL) post-training scheme that trains the world model on its own autoregressive rollouts rather than on ground-truth histories. We achieve this by adapting a recent contrastive RL objective for diffusion models to our setting
and show that its convergence guarantees carry over exactly.
%
Second, we design a training protocol that generates and compares multiple candidate variable-length futures from the same rollout state, reinforcing higher-fidelity predictions over lower-fidelity ones.
Third, we develop efficient, multi-view visual fidelity rewards that combine complementary perceptual metrics across camera views and are aggregated at the clip level for dense, low-variance training signal.
Fourth, we show that our approach establishes a new state-of-the-art for rollout fidelity on the DROID dataset, outperforming the strongest baseline on all metrics (e.g., LPIPS reduced by 14\% on external cameras, SSIM improved by 9.1\% on the wrist camera), winning 98\% of paired comparisons, and achieving an 80\% preference rate in a blind human study.


\end{abstract}
\section{Introduction}
\label{sec:intro}


\begin{figure}[tb]
  \centering
  \includegraphics[width=\textwidth]{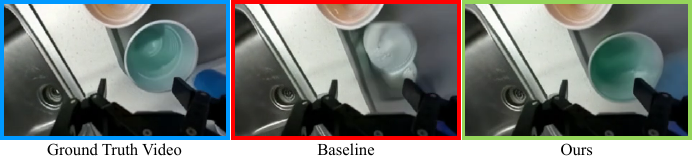}
    \caption{
  \textbf{Autoregressive video rollout quality.} We use an action-conditioned robot world model to generate multi-view image predictions from a single observed state. The ground-truth (\textcolor{myBlue}{blue}) is compared against the baseline (\textcolor{myRed}{red}) and our post-trained model {\bf PersistWorld} (\textcolor{myGreen}{green}). While the baseline accumulates error and destroys the object (cyan bowl) within seconds, our method maintains structural integrity and spatial consistency, establishing a new state-of-the-art in rollout fidelity.
  }
  \label{fig:teaser}
  \vspace*{-6mm}
\end{figure}


Action-conditioned video diffusion world models (WM) represent a transformative frontier for robot learning, offering the potential to simulate complex, human-centric tasks that are notoriously difficult to model with traditional physics-based engines. Recent work~\cite{ctrlworld} has demonstrated that finetuning pre-trained video diffusion backbones with action conditioning on large-scale robotics datasets~\cite{khazatsky2024droid} can yield world models that generate impressively faithful clips that are consistent with the robot's actions. By synthesizing high-fidelity visual rollouts conditioned on robot actions, these models can serve as scalable virtual environments for benchmarking and improving Vision-Language-Action (VLA) policies. 

However, realizing this potential requires generating \emph{long-horizon} autoregressive rollouts —  multiple seconds of coherent video — where each predicted clip feeds back as context for the next. This is precisely where current models break down. This phenomenon, known as \emph{exposure bias}~\cite{ranzato2015sequence}, arises from a train/test distribution mismatch, i.e., the model is trained to predict from \emph{ground-truth} history frames, but at deployment it must condition on its own previously generated outputs, which carry growing imperfections.
The consequences are rapid and severe. Within seconds of autoregressive generation, manipulated objects lose their structural identity — a bowl dissolves into an amorphous blob (see Fig.~\ref{fig:teaser}) — robot end-effectors drift from their commanded trajectories, and entire scene configurations decohere (see Fig.~\ref{fig:qualitative}). 

%
Reinforcement learning (RL) offers a natural framework for addressing this problem: by computing a training signal directly on the model's own autoregressive rollouts, it incentivizes consistent long-horizon generation rather than single-step accuracy. However, applying RL to diffusion models is challenging — standard policy-gradient methods require likelihoods that diffusion models do not provide, and backpropagating through the full denoising process is prohibitively expensive. Recent work~\cite{zheng2025diffusionnft} offers an elegant workaround by generating multiple candidate outputs, scoring them with a reward, and using the comparison to update the model — avoiding backpropagation through denoising entirely. However, this approach was developed for single-image generation and does not directly apply to our setting for two reasons. First, our robot world model uses a different type of denoising network, which means the theoretical guarantees need to be re-derived. Second, in image generation, the model can simply draw multiple independent samples from the same text prompt and compare them. In autoregressive video, there is no fixed prompt — each generation step builds on the previous output, creating an evolving shared state. This means we need a new mechanism for producing comparable candidates that can be meaningfully ranked against each other.
We address these challenges via the following contributions:

\begin{enumerate}
    \item \textbf{RL post-training for robot world models.} We introduce a post-training scheme that optimizes the world model directly on its own autoregressive rollouts rather than on ground-truth histories. We adapt a recent contrastive RL method for diffusion models~\cite{zheng2025diffusionnft} to our setting, where the denoising network directly predicts clean frames rather than intermediate noise, and show that convergence guarantees carry over exactly (Sec.~\ref{sec:rl}). 

    \item \textbf{A training protocol for autoregressive robot world models.} 
  Autoregressive robot world model has no fixed prompt from which multiple candidates can be drawn and compared. We observe that the model's accumulated history at any rollout step serves as a natural shared context from which multiple candidate continuations can be independently generated and ranked. By randomly sampling how deep into the rollout we branch these candidates, we expose training to both mild early-stage and severe late-stage error regimes (Sec.~\ref{sec:rollout}).
  

    \item \textbf{Multi-view visual rewards and task-relevant evaluation.} 
    We design clip-level rewards that combine complementary perceptual metrics across all three camera views, normalized so that the training signal reflects relative quality within each group of candidates. We further introduce object-centric and robot-centric masked evaluations that confirm our improvements come from better modeling of task-relevant dynamics rather than background preservation (Sec.~\ref{sec:rewards} and Sec.~\ref{sec:exp}).

    \item \textbf{State-of-the-art rollout quality.} Our post-trained model sets a new state-of-the-art on the DROID dataset~\cite{khazatsky2024droid} across all visual quality metrics, with the largest improvements on the wrist camera---the view most critical for capturing fine-grained object manipulation. In paired comparisons, our model outperforms the baseline on approximately 98\% of validation samples. A blind human preference study confirms these gains, with raters favoring our rollouts 80\% of the time (Sec.~\ref{sec:exp}).

     
\end{enumerate}

\section{Related work}
\label{sec:related_works}

\paragraphcustom{Video Diffusion Models.}
Diffusion models have emerged as a dominant paradigm for high-fidelity visual generation and have been successfully extended from images to videos by modeling space-time volumes with denoising objectives~\cite{ho2022video, ho2022imagen, singer2022make}. Latent video diffusion approaches adapt image diffusion backbones with temporal modules, enabling strong image-to-video and text-to-video generation while encoding rich visual priors about object motion, lighting, and physical plausibility~\cite{blattmann2023stable}. While these models produce impressive open-loop clips, long-horizon generation via autoregressive stitching of short segments causes errors in early frames to compound as the model conditions on its own imperfect outputs, leading to temporal drift and degradation. Our setting inherits this challenge in an action-conditioned, multi-view robotics regime.

\paragraphcustom{Robotic World Models and Video-Based Planning.}
World models have a long history in model-based RL as learned dynamics models enabling planning through imagined rollouts~\cite{ha2018world, hafner2019dream, hafner2019learning, hafner2023mastering}. Several works frame robot planning and evaluation directly as video generation~\cite{universalpolicies, Ko2023Learning, black2023zero, hu2025videopredictionpolicygeneralist}, and IRASim~\cite{zhu2025irasim} demonstrates that a trajectory-conditioned diffusion world model can serve as a policy evaluation. Large pre-trained video diffusion backbones have been adapted into controllable robot world models via action conditioning: Ctrl-World~\cite{ctrlworld} and WPE~\cite{quevedo2025worldgym}, both trained on DROID~\cite{khazatsky2024droid}, generate multi-view manipulation trajectories and can rank downstream policy performance; AVID~\cite{rigter2024avid} adapts pretrained video diffusion via a learned mask adapter without parameter access; and UWM~\cite{zhu2025unified} jointly models video and action diffusion. None of these directly address the training-inference mismatch under self-conditioned rollout.

\paragraphcustom{Post-Training Diffusion Models.}
Aligning diffusion models to downstream objectives has been studied through RL and preference optimization~\cite{black2023training, wallace2024diffusion, liu2025flow, xue2025dancegrpo, prabhudesaivader}. DiffusionNFT~\cite{zheng2025diffusionnft} proposes a negative-aware fine-tuning objective on the forward diffusion process, enabling efficient online RL updates without backpropagating through the denoising trajectory. DPPO~\cite{ren2024diffusion} applies RL fine-tuning to diffusion-based robot action policies; our work instead post-trains the world model itself. RLVR-World~\cite{wu2025rlvr} applies RL with verifiable rewards to improve world model transition quality, further evidencing that RL objectives outperform MLE for rollout fidelity. However, they work with token based models, whereas we work on improving video diffusion models.
Contemporary to our work,~\cite{wang2026worldcompass} applies contrastive RL post-training~\cite{zheng2025diffusionnft} to a camera-pose-conditioned world model, targeting improvements in camera pose following and visual fidelity—including a prefix rollout strategy in which the model's own outputs serve as context for subsequent clip generation. Our work shares this core motivation and training paradigm, but differs in several respects. We adopt a randomized prefix horizon during post-training rather than a fixed schedule, which we find better captures the distribution of compounding errors at test time. Our setting targets dynamic manipulation scenes in a multi-view robotic world model and we introduce a novel adaptation of the contrastive RL objective to the $\mathbf{x}_0$-prediction parameterization used by some robot WMs. We further design visual rewards that are efficient and scalable for the multi-view robot manipulation set-up. We find that compared to the HPSv3~\cite{ma2025hpsv3} reward that is used by~\cite{wang2026worldcompass}, our rewards result in much better performance. Finally, we additionally validate our approach through quantitative evaluation on robot-centric metrics and a human preference study.

\paragraphcustom{Exposure Bias and the Rollout Gap.}
The mismatch between teacher-forced training and self-conditioned inference is a longstanding problem in sequence generation~\cite{bengio2015scheduled, ranzato2015sequence}, and has been studied specifically for diffusion models~\cite{ning2023input}. In video generation the effect is amplified: small per-frame errors accumulate over long rollouts, degrading coherence and limiting world model utility for evaluation and simulation. We address this \emph{rollout gap} with RL post-training that directly exposes the model to its own generated histories during training.

\paragraphcustom{Large-Scale Robot Datasets.}
Large-scale robot datasets enable both training data-hungry video world models and the held-out ground-truth trajectories our reward computation relies on. DROID~\cite{khazatsky2024droid} provides diverse multi-camera manipulation demonstrations, and Open X-Embodiment~\cite{o2024open} aggregates demonstrations across many embodiments and institutions—making dataset-driven reward evaluation feasible without human preference labels.
\section{Improving Robot World Models with Reinforcement Learning}
\label{sec:method}

\begin{figure}[t!]
  \centering
  \includegraphics[width=\textwidth]{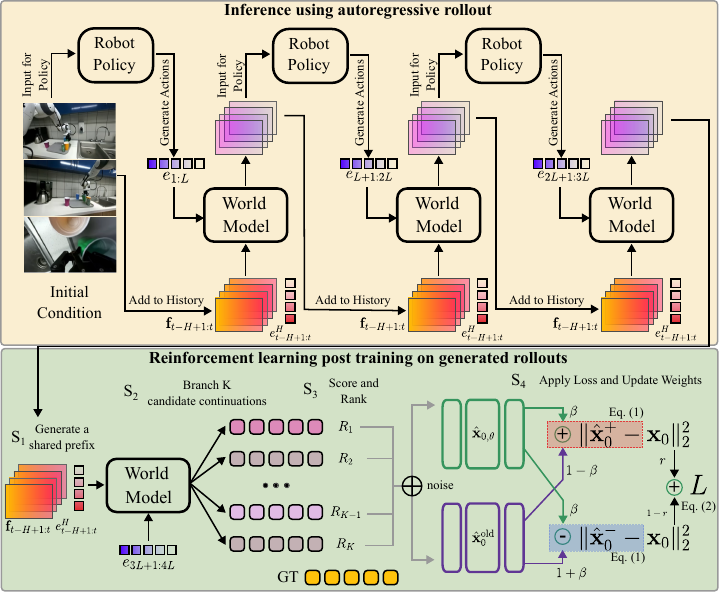}
   \caption{\small \textbf{Overview of our method:} 
   \textit{(Top)} \textbf{Autoregressive inference:} A robot policy generates actions fed to the world model, which produces multi-view frames that are appended to the history buffer and condition the next generation step. \textit{(Bottom)} \textbf{RL post-training:} \textbf{(S1)}~A shared variable-length prefix is rolled out autoregressively from a ground-truth initial condition. \textbf{(S2)}~$K$ independent candidate continuations are branched from the frozen prefix state. \textbf{(S3)}~Candidates are scored against ground-truth using multi-view perceptual rewards. \textbf{(S4)}~Reward weights $r$ scale implicit positive/negative $\mathbf{x}_0$ predictions used in contrastive model updates via loss $L$.}
  \label{fig:overview}
  \vspace*{-4mm}
\end{figure}

Action-conditioned video diffusion world models are trained to predict the next chunk of video frames given a clean, ground-truth history — a setup that works well in isolation, but breaks down the moment the model is deployed autoregressively over longer horizons. The central challenge we address is the following: because the model has never seen imperfect history at training time, any error introduced at one step propagates forward and compounds at the next, causing rollout quality to degrade rapidly. Resolving this requires more than better data or longer training; it requires a different approach to training — one that explicitly optimizes the model under the same auto-regressive, self-conditioned regime in which it operates at test time.

This section develops our approach in three stages. We begin by characterizing the train/test distribution mismatch and why standard training cannot resolve it (Sec.~\ref{sec:problem}). We then formulate post-training as an online reinforcement learning (RL) problem and derive a tractable objective by adapting a contrastive forward-process training designed for velocity-prediction flow-matching models~\cite{zheng2025diffusionnft} — to the EDM~\cite{karras2022elucidatingdesignspacediffusionbased} $x_0$-prediction parameterization of Ctrl-World's SVD backbone; we show that the branch construction and policy-improvement guarantees carry over exactly, with no $\sigma$-dependent correction terms (Sec.~\ref{sec:rl}). Finally, we address two concrete design challenges: how to construct a group-relative training signal from autoregressive video rollouts (Sec.~\ref{sec:rollout}), and how to define rewards that faithfully assess multi-view, multi-step visual quality (Sec.~\ref{sec:rewards}).

\subsection{The Closed-Loop Gap in Autoregressive World Models}
\label{sec:problem}

\paragraphcustom{Model overview.} 
An autoregressive robot world model, such as Ctrl-World~\cite{ctrlworld}, operates in a loop: at each step, it receives a history buffer of recent frame latents encoding the visual context so far, together with past and future robot actions, and generates the next chunk of future frames. The generated frames are then encoded and appended to the history buffer, which conditions the next generation step.
%
In detail, at each autoregressive step, the model receives three inputs: (i) a history buffer $\mathbf{f}_{t-H+1:t}$ of $H{=}6$ recent frame latents that encode the visual context of what has transpired so far; (ii) the corresponding robot end-effector (EEF) poses $\mathbf{e}^H_{t-H+1:t}$; and (iii) a sequence of future EEF pose targets $\mathbf{e}_{t+1:t+L} \in \mathbb{R}^{L \times 7}$. From these, the model simultaneously generates $L{=}5$ future frames across three camera views — two external views and one wrist-mounted view.

\paragraphcustom{The closed-loop gap.} Training follows a standard diffusion objective under teacher forcing: given a clip of length $H + L$ from the dataset, the model learns to denoise the final $L$ frames conditioned on the preceding $H$ \emph{ground-truth} frames as history. This produces a reliable training signal, but installs a structural mismatch with deployment. At test time, no ground-truth history is available; each generated clip is encoded and appended to the rolling history buffer, which then conditions the next denoising pass. The model must now condition on its own previous outputs — inputs it was never trained to handle.
The consequence is an error compounding loop. A minor spatial or temporal inaccuracy in clip $t$ corrupts the latents stored in the history buffer. These corrupted latents condition clip $t+1$ with increased error, which in turn corrupts the history for clip $t+2$. Within seconds, generated scenes decohere: object configurations blur, robot state diverges from the commanded trajectory, and scene identity dissolves. Rollouts beyond a few seconds become unreliable as surrogates for real-world execution — precisely the use-case that makes world models valuable.
This is not a data sufficiency problem. No amount of teacher-forced training gives the model incentive to be robust to its own imperfect history, because imperfect history is absent from the training distribution by construction. What is needed is a training signal computed directly from the model's own autoregressive outputs — one that rewards coherent closed-loop generation and penalizes compounding drift.

\subsection{Online RL Post-Training via Reward-Contrasted Denoising}
\label{sec:rl}

Online reinforcement learning offers a principled solution: generate rollouts autoregressively, evaluate them against held-out ground truth, and update the model toward higher-fidelity outputs — with the training distribution defined by the model's own production rather than teacher-forced ground truth. Because the reward is computed on self-generated frames, the training signal inherently reflects the closed-loop statistics of deployment. The challenge is making this compatible with diffusion models.


\paragraphcustom{Reward-conditioned forward-process training.}
 We address this challenge by recasting policy improvement as contrastive denoising~\cite{zheng2025diffusionnft}:
rather than estimating reverse-process likelihoods, we generate a group of candidate outputs, score them with a reward, and encode the relative quality signal directly into the denoising loss — reinforcing what the model produces for high-reward candidates and penalizing what it produces for low-reward ones. 
The contrastive denoising approach~\cite{zheng2025diffusionnft} was originally derived for velocity-prediction flow-matching models such as SD3~\cite{sd3}.  However, some world models ---
including Ctrl-World~\cite{ctrlworld}, which we build upon --- employ an $x_0$-prediction
parameterization, where the network directly estimates the clean data $x_0$ rather than a velocity
field. We show that the contrastive denoising framework transfers naturally to this setting: because the mapping
from network output to clean-data prediction is affine, the contrastive objective construction
and its policy-improvement guarantees carry over exactly, with no additional correction terms.

The full derivation is in the \textbf{Appendix~\ref{app:diffusionnft-x0}}. 
The resulting $x_0$-adapted objective takes the following form. Let $\hat{\mathbf{x}}_{0,\theta}$ denote the model's clean data ($x_0$) prediction and $\hat{\mathbf{x}}_0^{\mathrm{old}}$ a frozen exponential moving average (EMA) copy serving as the reference policy. For a candidate with normalized reward weight $r \in [0,1]$ and mixing coefficient $\beta$, we construct implicit positive and negative clean data ($x_0$) predictions:
\begin{equation}
\label{eq:x_predictions}
  \hat{\mathbf{x}}_{0}^+ = (1-\beta)\,\hat{\mathbf{x}}_0^{\mathrm{old}} + \beta\,\hat{\mathbf{x}}_{0,\theta}, \qquad
  \hat{\mathbf{x}}_{0}^- = (1+\beta)\,\hat{\mathbf{x}}_0^{\mathrm{old}} - \beta\,\hat{\mathbf{x}}_{0,\theta},
\end{equation}
and minimize the reward-weighted denoising loss:
\begin{equation}
  \mathcal{L}(\theta) = \mathbb{E}\Bigl[r\,\|\hat{\mathbf{x}}_{0}^+ - \mathbf{x}_0\|_2^2 + (1-r)\,\|\hat{\mathbf{x}}_{0}^- - \mathbf{x}_0\|_2^2\Bigr].
\label{eq:diffusionnft-x0-loss}
\end{equation}
Intuitively, the difference $\hat{\mathbf{x}}_{0,\theta} - \hat{\mathbf{x}}_0^{\mathrm{old}}$
defines the direction in which the current model has drifted from the frozen reference. The
positive branch $\hat{\mathbf{x}}_0^+$ \emph{extrapolates} along this direction: it takes
the reference prediction and moves it toward the current model by a factor $\beta$, amplifying
whatever changes the model has learned. Conversely, the negative branch
$\hat{\mathbf{x}}_0^-$ \emph{reverses} this direction, constructing a counterfactual
prediction that moves away from the current model's output. The loss in
Eq.~\ref{eq:diffusionnft-x0-loss} then uses the reward weight $r$ to interpolate between fitting
$\hat{\mathbf{x}}_0^+$ (reinforcing the model's current direction for high-reward samples)
and fitting $\hat{\mathbf{x}}_0^-$ (repelling the model from its own predictions for
low-reward samples). The mixing coefficient $\beta$ controls the strength of this
amplification: larger $\beta$ produces a stronger reinforcement signal but risks
destabilizing training, while $\beta \to 0$ recovers standard supervised denoising against
the reference. Note that this formulation requires only the clean generated samples and the
reference predictions --- it avoids backpropagating through the denoising chain entirely,
making it compatible with any black-box sampler.
\subsection{Adapting Group-Relative Training to Autoregressive Video}
\label{sec:rollout}

The RL formulation above addresses how to update a diffusion model given reward-scored samples. Applying it to an autoregressive world model raises a second, distinct challenge: the formulation assumes a natural grouping structure — a shared conditioning input from which multiple independent candidate outputs are drawn. In image generation, this structure is straightforward: sample $K$ independent images from the same text prompt and compare them by reward. In autoregressive video generation, no such fixed prompt exists. Each generation step produces a clip that modifies the shared history buffer, which then conditions the next step; candidate clips are not independent draws from a common condition, but sequential extensions of an evolving shared state.

The key observation is that the history buffer state immediately before any generation step plays exactly the role of the prompt in the group-relative setting: it is the accumulated context from which distinct candidate continuations can be independently branched. Freezing this buffer state and sampling $K$ independent candidate next clips from it yields a group that shares a common context, enabling meaningful reward-based comparison and contrastive training. 
This recovers the shared-context / independent-response structure required by group-relative objectives.
%
We realize this structure through the following rollout protocol at each training step:
\begin{enumerate}
  \item[$\text{S}_1$:] \textbf{Generate a shared prefix.} Starting from a single ground-truth observation — with the history buffer backfilled by replicating its encoded latent — we autoregressively generate $P$ consecutive clips, feeding the model's own outputs back as history at each step. This mirrors closed-loop deployment and produces a history buffer state that has been corrupted by the model's own accumulated errors. The prefix length is sampled as $P \sim \mathrm{Unif}\{0, 1, \ldots, 9\}$, exposing training to the full spectrum of rollout positions — from early steps where the buffer is nearly clean to late steps where compounding drift is severe.

  \item[$\text{S}_2$:] \textbf{Branch $K$ candidate continuations.} From the frozen prefix history buffer, we independently sample $K{=}16$ candidate next segments. Each candidate is a short autoregressive sequence of $F$ chunks: the model generates $L$ frames across all three views simultaneously, encodes and appends them to a private copy of the history buffer, and repeats for $F$ steps. Each candidate follows its own distinct stochastic trajectory from the shared context.

  \item[$\text{S}_3$:] \textbf{Score and rank.} A visual reward $R_t^{(k)}$ is computed for each candidate by comparing its generated frames against held-out ground-truth frames across all three camera views (Sec.~\ref{sec:rewards}). Rewards are group-normalized over the $K$ candidates to form relative advantages, removing the influence of absolute reward scale at different rollout positions.

  \item[$\text{S}_4$:] \textbf{Update the model.} The group-normalized reward weights are used to scale the positive and negative denoising losses (Eq.~\ref{eq:diffusionnft-x0-loss}), and the model is updated via gradient descent. Only LoRA adapters and the action encoder receive gradient updates; the backbone is frozen.
\end{enumerate}

The variable prefix length serves two purposes. It ensures the model is optimized to maintain quality across the full rollout depth, not only at short horizons. It also exposes the update to diverse history buffer corruption profiles — from lightly drifted early-step buffers to heavily degraded late-step ones — preventing overfitting to any single error regime.

\subsection{Visual Rewards for Multi-View Video Clips}
\label{sec:rewards}

Defining an effective reward for autoregressive robot video generation requires carefully considering what to measure and how to aggregate it. The signal must be dense enough to be informative at each training step, directly tied to perceptual quality rather than proxy statistics, and consistent across the three camera views, which provide complementary perspectives on the manipulated scene. A reward aggregated across an entire long-horizon rollout would be high-variance and would make credit assignment to specific generations difficult; we therefore score at the granularity of individual clips.

For each candidate clip at time $t$, we compare generated frames $\hat{\mathbf{x}}^{(v)}_{t+1:t+L}$ against the held-out ground-truth frames $\mathbf{x}^{(v)}_{t+1:t+L}$ for each view $v \in \mathcal{V} = \{\mathrm{wrist},\, \mathrm{ext}_1,\, \mathrm{ext}_2\}$. Per-frame metrics are first averaged temporally over the $L$ frames of the clip:
\begin{equation}
  \overline{m}^{(v)}_t = \frac{1}{L}\sum_{\ell=1}^{L} m\!\left(\hat{\mathbf{x}}^{(v)}_{t+\ell},\, \mathbf{x}^{(v)}_{t+\ell}\right), \quad m \in \{\mathrm{LPIPS},\, \mathrm{SSIM},\, \mathrm{PSNR}\},
\end{equation}
and then averaged equally across the three views:
\begin{equation}
  \overline{m}_t = \frac{1}{3}\!\left(\overline{m}^{(\mathrm{wrist})}_t + \overline{m}^{(\mathrm{ext}_1)}_t + \overline{m}^{(\mathrm{ext}_2)}_t\right).
\end{equation}
We use three complementary metrics to capture distinct failure modes. LPIPS~\cite{zhang2018unreasonable} measures perceptual similarity in deep feature space, penalizing structural distortions even when pixel values are numerically close. SSIM~\cite{ssim} captures luminance, contrast, and local structural fidelity over spatial patches. PSNR~\cite{psnr} provides a global signal-to-noise measure that is sensitive to large pixel deviations, acting as a coarse indicator of catastrophic scene drift. Using all three produces a reward robust to the blind spots of any individual metric. LPIPS captures perceptual distortions invisible to pixel-level metrics, SSIM is sensitive to local structural changes, and PSNR flags large-scale pixel drift; combining them guards against failure modes that any single metric would miss.
The per-view, per-metric averages are combined into a single scalar reward:
\begin{equation}
  R_t = -w_{\mathrm{LPIPS}}\,\overline{\mathrm{LPIPS}}_t + w_{\mathrm{SSIM}}\,\overline{\mathrm{SSIM}}_t + w_{\mathrm{PSNR}}\,\overline{\mathrm{PSNR}}_t,
\end{equation}
where LPIPS is negated (it is lower-better), and the weights $w$ are set to bring the three components to a comparable numerical scale (values in Sec.~\ref{sec:exp}).

\paragraph{Group normalization.}
Because absolute reward values vary significantly across rollout positions — early clips score much higher than late ones — we normalize rewards within each group to focus the training signal on relative quality differences.
The per-candidate rewards $R_t^{(k)}$ are group-normalized over the $K$ candidates via z-score normalization:
\begin{equation}
  A^{(k)} = \frac{R^{(k)} - \mu_R}{\sigma_R + \epsilon},
  \quad \mu_R = \frac{1}{K}\sum_{k=1}^{K}R^{(k)}, \quad \sigma_R = \mathrm{std}_{k}(R^{(k)}).
\end{equation}
This converts absolute reward values into relative rankings within the group, removing the confound of reward scale variation across rollout positions. The z-scored advantages are clipped to $[-1, 1]$ and linearly rescaled to the $[0,1]$ range required by Eq.~\ref{eq:diffusionnft-x0-loss}:
\begin{equation}
\label{eq:scaled_r}
  r^{(k)} = \frac{\mathrm{clip}(A^{(k)},\,-1,\,1) + 1}{2}.
\end{equation}
This normalization encourages the model to discriminate between better and worse continuations from the same context, and keeps the gradient magnitude bounded regardless of the absolute level of visual quality at any given rollout position.
\section{Experiments}
\label{sec:exp}

\paragraphcustom{Implementation details.}
We use the pre-trained Ctrl-World model~\cite{ctrlworld} as our base. For the proposed post-training, we apply LoRA~\cite{hu2021loralowrankadaptationlarge} adapters to the UNet backbone (rank $r{=}64$, $\alpha{=}64$) and additionally finetune the action encoder; all other parameters (other UNet layers, VAE, etc.) are frozen. We train for 6{,}000 steps with learning rate $1{\times}10^{-4}$ using the Muon optimizer~\cite{jordan2024muon} with the batch size $64$ and group size $K{=}16$. Additionally, we subsample the group elements by taking the 10 most informative samples (top-5 and bottom-5 ordered by the reward) per update step. Reward weights are set to: $w_{\mathrm{LPIPS}}{=}w_{\mathrm{SSIM}}{=}1$ and $w_{\mathrm{PSNR}}{=}\tfrac{1}{32}$, with the $\tfrac{1}{32}$ factor bringing PSNR into a comparable numerical range with SSIM~$\in[0,1]$. From our ablations this leads to the best average performance across all metrics. The model is trained on $8$~NVIDIA H200 GPUs for $2$ days. 
Primary training cost is the autoregressive rollout generation, common in prior works~\cite{self-forcing, rolling-forcing}.
Overall, we found that our choices are a conservative default 
and one can tune even more to obtain better performance; for example, see Table~\ref{tab:abl-group} in Appendix. For more details on ablations, see Appendix~\ref{app:ablations}.

\paragraphcustom{Dataset.}
We evaluate on the DROID dataset~\cite{khazatsky2024droid}, a large-scale robot manipulation dataset collected on a Franka Emika Panda robot across a diverse set of tabletop environments. DROID comprises teleoperated demonstrations across a wide variety of everyday manipulation tasks and uses a standardized three-camera setup (two external cameras and one wrist-mounted camera). We use Ctrl-World's held-out validation split for all quantitative evaluations. 

\paragraphcustom{Autoregressive rollout quality evaluation.}
We evaluate autoregressive rollout quality on pre-recorded trajectories from the validation split. Starting from a single observed state (frames from all cameras + robot EEF pose), we generate 14 consecutive clips ($14 \times L{=}70$ frames, covering approximately 11\,s at 5\,Hz) using the autoregressive procedure from Sec.~\ref{sec:rollout}. We then compare the generated frames to the corresponding ground-truth frames from the dataset using SSIM~\cite{ssim}, PSNR~\cite{psnr}, and LPIPS~\cite{zhang2018unreasonable}. We report metrics separately for external cameras and the wrist camera, as these views capture qualitatively different aspects of the scene.

Our model establishes a new state-of-the-art for autoregressive rollout quality on the DROID dataset, consistently outperforming all baselines—WPE~\cite{quevedo2025worldgym}, IRASim~\cite{zhu2025irasim}, and Ctrl-World~\cite{ctrlworld}—across every metric (Table~\ref{tab:visual-results}).
Compared to Ctrl-World baseline, our approach achieves significant gains on external cameras, improving PSNR by 1.40~dB and reducing LPIPS by 14.0\%. These margins widen considerably against WPE and IRASim, where we see PSNR improvements of 4.09~dB and 3.06~dB and LPIPS reductions of 46.6\% and 40.2\%, respectively.
The most pronounced gains occur on the wrist camera (SSIM +9.1\%, PSNR +1.59~dB). This indicates that our closed-loop-aware post-training specifically excels at capturing the fine-grained object contact and hand-eye coordination details critical for downstream policy evaluation. Qualitatively, Fig.~\ref{fig:qualitative} confirms these improvements, while Fig.~\ref{fig:short} illustrates a consistent distribution shift toward higher-fidelity generations. A 1-to-1 paired comparison reveals that our model outperforms the baseline in $\sim$98\% of validation samples, demonstrating the consistency of our gains across the entire dataset. Finally, Fig.~\ref{fig:temporal-evolution} analyzes performance over extended rollouts; while both models naturally degrade as the horizon increases, our method maintains significantly higher fidelity over time.

\paragraphcustom{Comparison to Self-Forcing style baseline.}
Recent works such as Self-Forcing~\cite{self-forcing, rolling-forcing} are not directly applicable to our case since they primarily distill from a (usually larger) bidirectional teacher, which is not available in our case. However, to measure the importance of our RL objective over the training on corrupted history, we post-trained the model on its own rollout history but ask it to regress the ground truth latents. On the validation set with rollout length of 10 interactions, the visual metrics favor ours: wrist \textbf{0.721}/\textbf{20.40}/0.338 vs SF 0.710/20.17/0.338; external \textbf{Ours }\textbf{0.872}/\textbf{25.52}/\textbf{0.116} vs 0.868/25.35/0.119. 

\begin{table}[tb]
\centering
\caption{
\textbf{Visual quality metrics for 14-step autoregressive rollouts ($\approx$11~s) on the DROID validation split.} Values represent averages over the full rollout duration. Results marked with $^{*}$ are from~\cite{ctrlworld}; $^{\dagger}$ indicates our reproduction.
}
\label{tab:visual-results}
\small
\begin{tabular}{llccc}
\toprule
\multirow{2}{*}{\begin{tabular}[c]{@{}l@{}}Evaluated\\Cameras\end{tabular}} & 
\multirow{2}{*}{Model} & 
\multicolumn{2}{c}{Pixel/Structure} & 
\multicolumn{1}{c}{Perceptual} \\
\cmidrule(lr){3-4}\cmidrule(lr){5-5}
 & & SSIM $\uparrow$ & PSNR $\uparrow$ & LPIPS $\downarrow$ \\
\midrule
\multirow{5}{*}{\begin{tabular}[c]{@{}l@{}}External\\Camera\end{tabular}} 
 & WPE$^{*}$ & 0.77 & 20.33 & 0.131 \\
 & IRASim$^{*}$ & 0.77 & 21.36 & 0.117 \\
 & Ctrl-World$^{*}$ & 0.83 & 23.56 & 0.091 \\
 & Ctrl-World$^{\dagger}$ & 0.84 & 23.02 & 0.081 \\
 & \textbf{Ours} & \textbf{0.86} & \textbf{24.42} & \textbf{0.070} \\
\midrule
\multirow{2}{*}{\begin{tabular}[c]{@{}l@{}}Wrist\\Camera\end{tabular}} 
 & Ctrl-World$^{\dagger}$ & 0.62 & 17.80 & 0.310 \\
 & \textbf{Ours} & \textbf{0.67} & \textbf{19.39} & \textbf{0.277} \\
\bottomrule
\end{tabular}
\vspace{-3mm}
\end{table}

\begin{figure}[tb]
  \centering
  \includegraphics[width=0.48\textwidth]{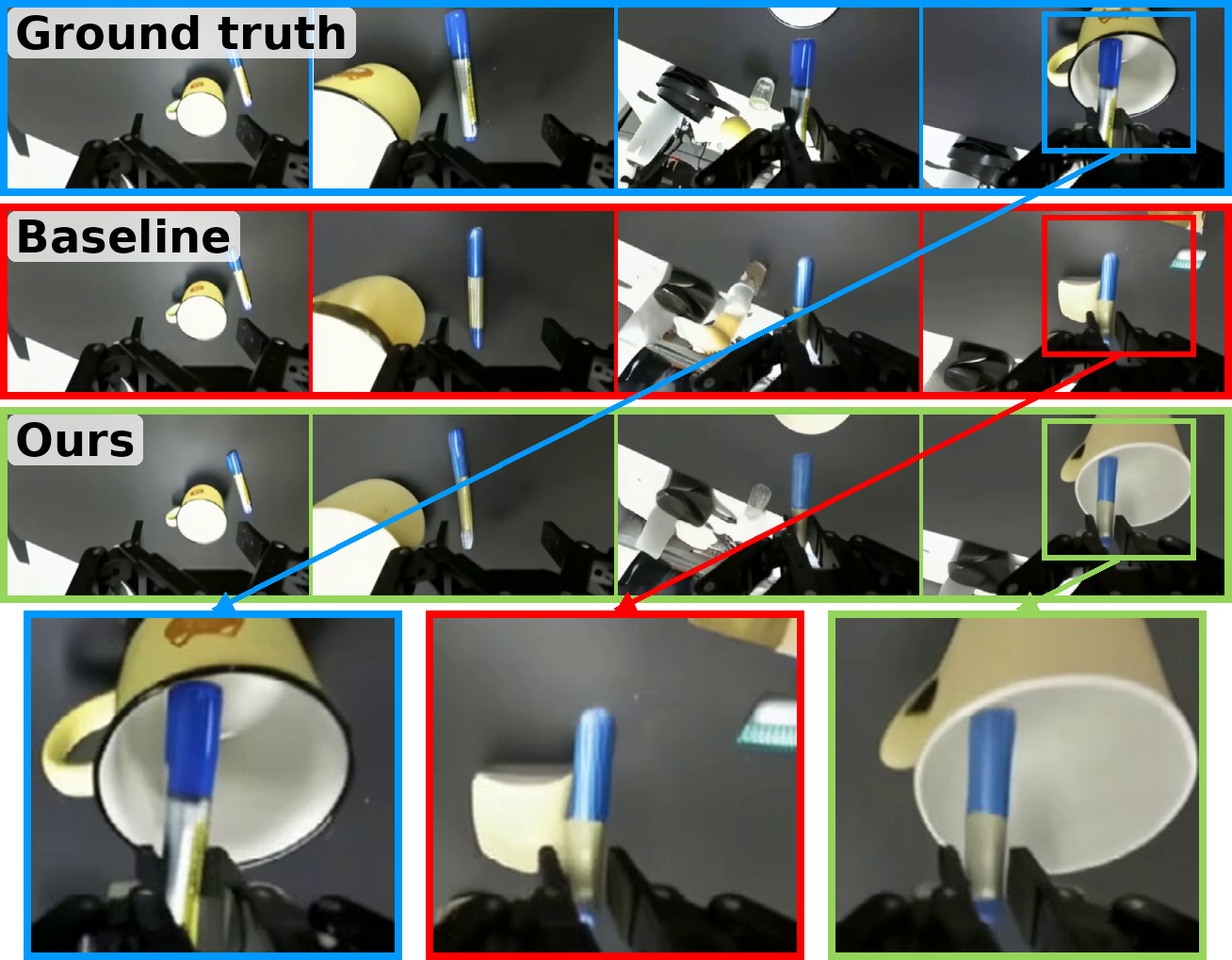}
  \includegraphics[width=0.48\textwidth]{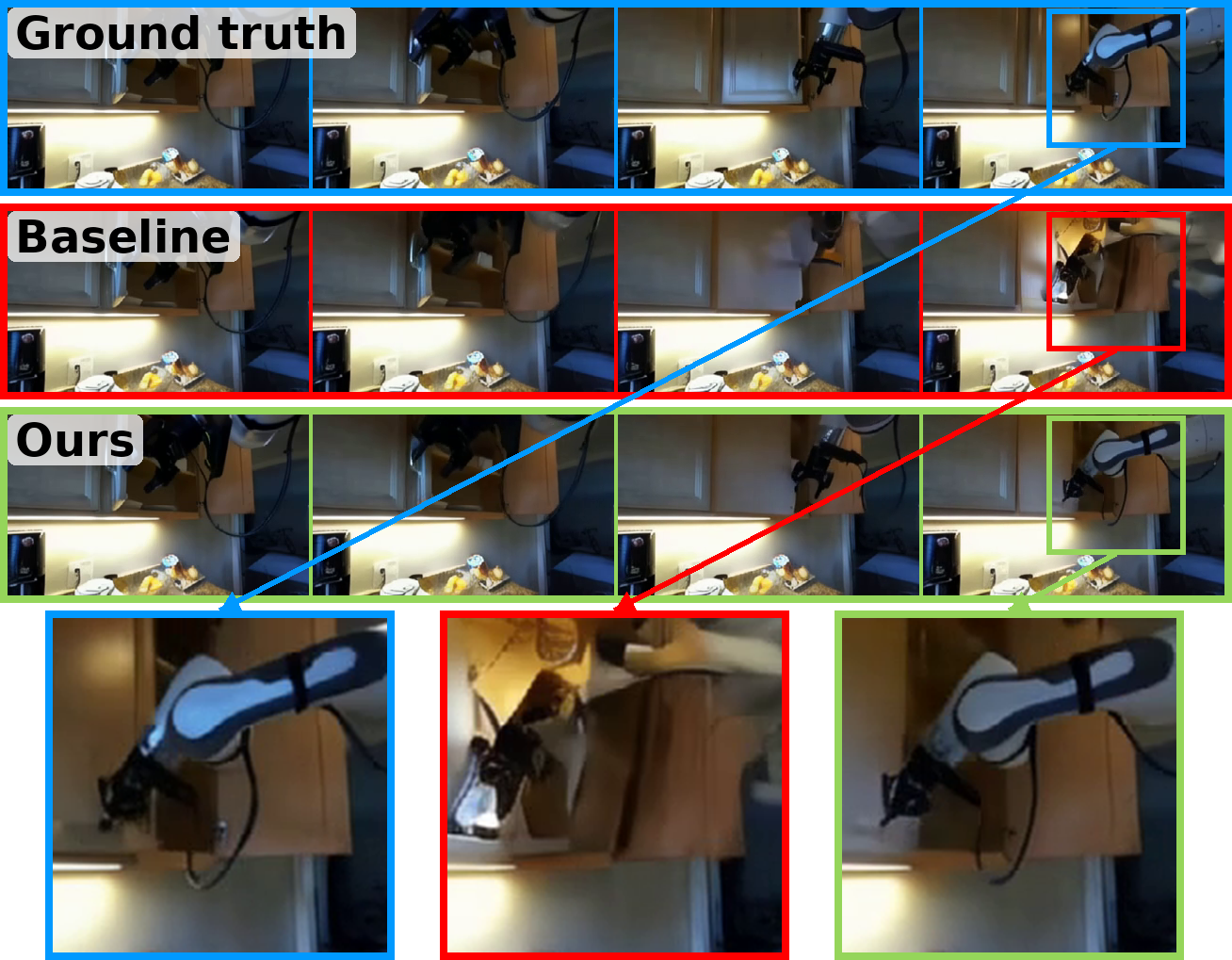}
  \caption{
  \textbf{Qualitative comparison of autoregressive rollout stability.} We compare long-horizon (11~s) generations from the \textcolor{myRed}{baseline}~\cite{ctrlworld} against our \textcolor{myGreen}{PersistWorld} for the wrist camera. 
    \textbf{Left: Object-centric fidelity.} The baseline model suffers from rapid decoherence; as errors compound in the history buffer, manipulated objects like the cup lose their structural identity and dissolve into amorphous textures. In contrast, our method maintains the spatial consistency and structural integrity of the object throughout the rollout. 
   \textbf{Right: Robot-centric consistency.} The baseline exhibits significant  \textbf{robot decoherence}, where the generated robot arm loses their geometric structure. Our approach maintains structural persistence. {\bf Please see additional video results on the associated project page.}
    }
    \vspace{-5mm}
  \label{fig:qualitative}
\end{figure}

\begin{figure}[tb]
  \centering
  \begin{subfigure}{0.32\linewidth}
  \includegraphics[width=\textwidth]{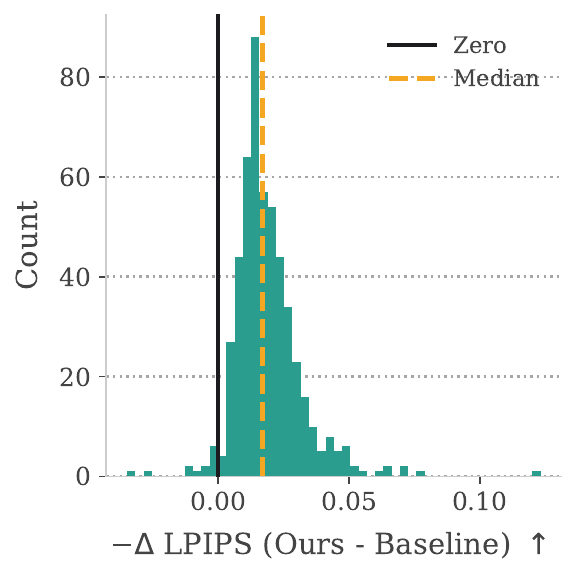}
    \label{fig:short-a}
  \end{subfigure}
  \hfill
  \begin{subfigure}{0.32\linewidth}
    \includegraphics[width=\textwidth]{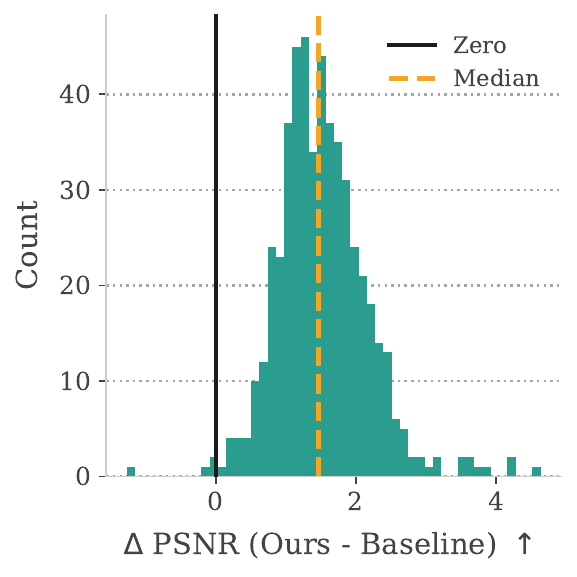}
    \label{fig:short-b}
  \end{subfigure}
  \hfill
  \begin{subfigure}{0.32\linewidth}
    \includegraphics[width=\textwidth]{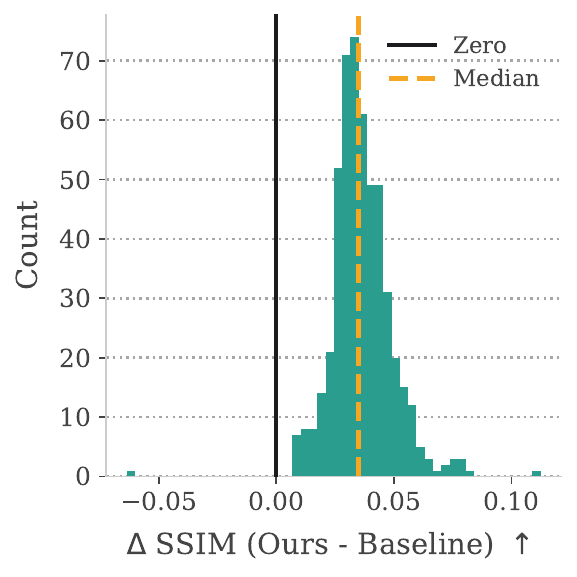}
    \label{fig:short-c}
  \end{subfigure}
  \vspace*{-6mm}
  \caption{$\Delta_{\text{metric}}$ of paired videos from the validation dataset. On $1-1$ paired comparison, our PersistWorld world model is better than the baseline on $\sim98\%$ of the sample ($p < 10^{-6}$). }
  \label{fig:short}
\end{figure}

\begin{figure}[tb]
  \centering
  \begin{subfigure}{0.32\linewidth}
  \includegraphics[width=\textwidth]{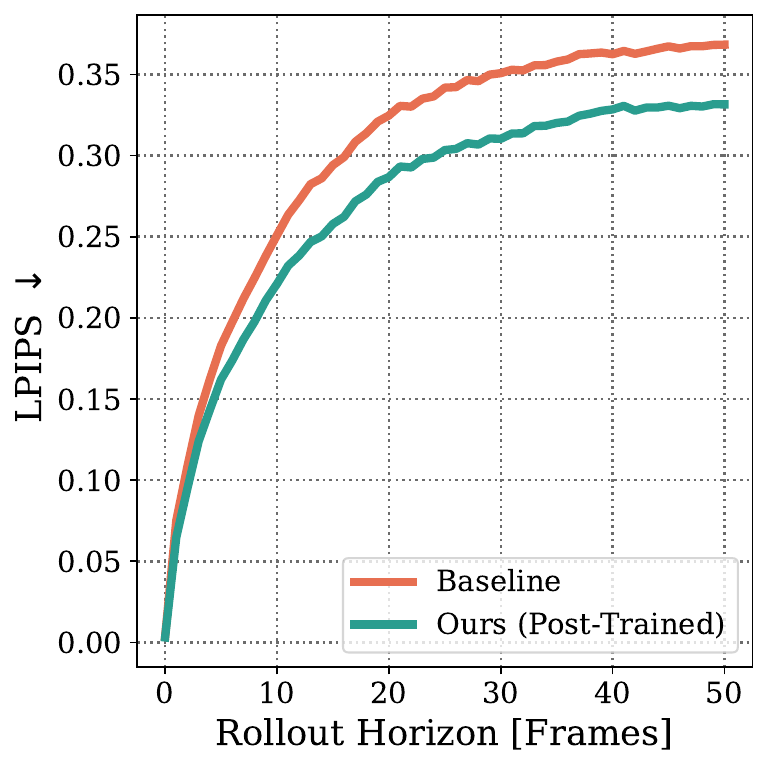}
  \end{subfigure}
  \hfill
  \begin{subfigure}{0.32\linewidth}
    \includegraphics[width=\textwidth]{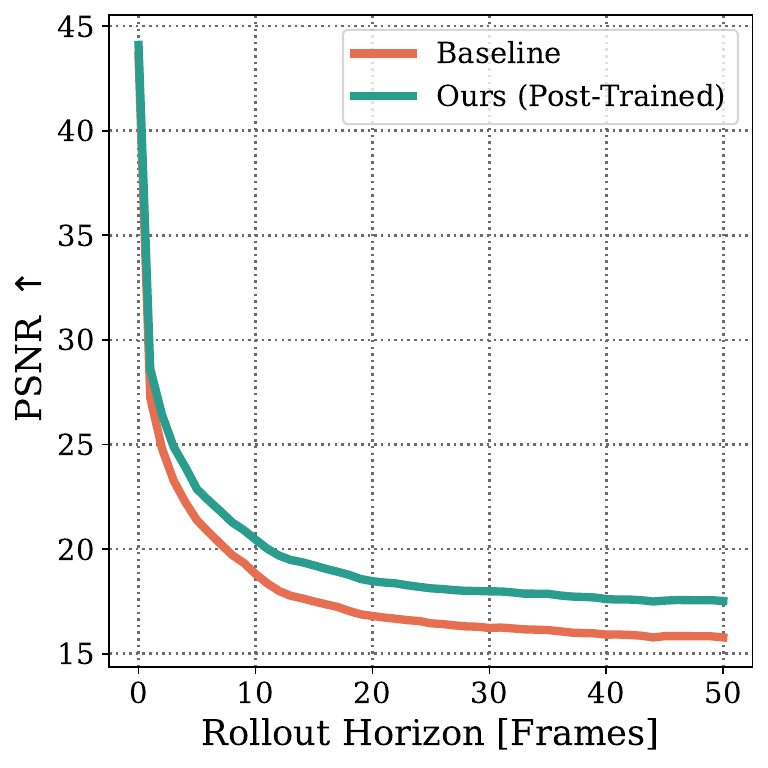}
  \end{subfigure}
  \hfill
  \begin{subfigure}{0.32\linewidth}
    \includegraphics[width=\textwidth]{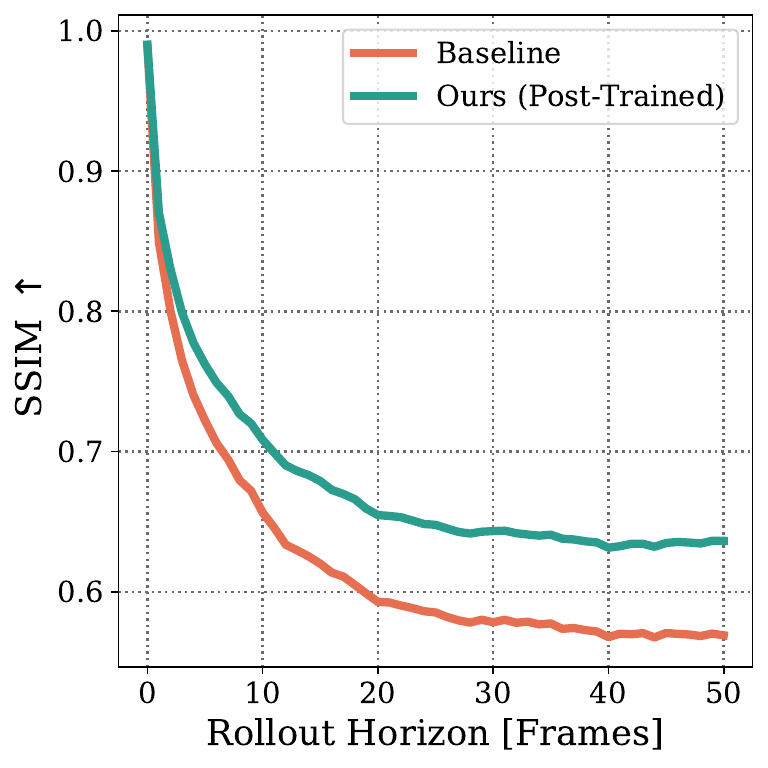}
  \end{subfigure}
  \vspace*{-2mm}
  \caption{
      \textbf{Temporal evolution of wrist camera metrics.} While both models exhibit natural degradation over longer horizons (x-axis), our post-trained model, PersistWorld (green), consistently maintains higher fidelity and slower error accumulation compared to the baseline (orange). Specifically, our method preserves a higher PSNR and SSIM while suppressing LPIPS drift, effectively extending the stable prediction horizon for complex, fine-grained interactions. See Fig.~\ref{fig:external-cam-temporal-evolution} in \textbf{Appendix.~\ref{app:additional-results}} for external camera results.
  }
  \label{fig:temporal-evolution}
  \vspace*{-4mm}
\end{figure}

\paragraphcustom{Object- and Robot-Centric Evaluation.} General-purpose world models often achieve high full-frame scores by over-optimizing for static background preservation while failing to capture the complex dynamics of manipulated objects. To evaluate task-relevant fidelity, we isolate the foreground using RoboEngine~\cite{yuan2025roboengineplugandplayrobotdata} to segment interacting objects and the robot arm. Computing metrics on these masked regions provides a rigorous measure of spatial and control consistency, which is more critical for downstream policy learning than raw background reconstruction.
Table~\ref{tab:visual-results-masked} confirms that our model's gains are concentrated on these task-critical regions. On object-masked pixels, our improvements are even more pronounced than full-frame results: external camera LPIPS drops by $16.3\%$ (vs. $14.0\%$ full-frame), while wrist-camera SSIM improves by $4.1\%$. We observe similar trends for robot-centric metrics, with PSNR increasing by $1.63$~dB and $1.74$~dB for external and wrist views, respectively. These results demonstrate that our training objective successfully captures the intricate dynamics of robot-object interactions rather than relying on incidental background fidelity.

\begin{table}[tp]
\centering
\caption{
\textbf{Masked visual metrics for 14-step autoregressive rollouts ($\approx$11~s)}. We isolate object-only and robot-only pixels to evaluate task-relevant spatial and control consistency. Our model demonstrates superior fidelity in these dynamic regions compared to baselines, confirming that performance gains are driven by accurate interaction modeling rather than background reconstruction. All metrics are averaged over the full rollout horizon.
}
\label{tab:visual-results-masked}
\begin{tabular}{@{}cccccccc@{}}
\toprule
\multirow{2}{*}{\begin{tabular}[c]{@{}c@{}}Evaluated\\Cameras\end{tabular}} &
\multirow{2}{*}{Model} &
\multicolumn{3}{c}{Object-Only} &
\multicolumn{3}{c}{Robot-Only} \\
\cmidrule(lr){3-5}\cmidrule(lr){6-8}
& &
SSIM $\uparrow$ & PSNR $\uparrow$ & LPIPS $\downarrow$ &
SSIM $\uparrow$ & PSNR $\uparrow$ & LPIPS $\downarrow$ \\
\midrule
\multirow[c]{2}{*}{\begin{tabular}[c]{@{}c@{}}External\\Camera\end{tabular}} &
Ctrl-World$^{\dagger}$ & 0.88 & 22.25 & 0.025 & 0.82 & 17.62 & 0.039 \\
& \textbf{Ours} & \textbf{0.89} & \textbf{23.60} & \textbf{0.021} & \textbf{0.86} & \textbf{19.25} & \textbf{0.033}\\
\midrule
\multirow[c]{2}{*}{\begin{tabular}[c]{@{}c@{}}Wrist\\Camera\end{tabular}} &
Ctrl-World$^{\dagger}$ & 0.73 & 18.52 & 0.088 & 0.83 & 25.50 & 0.027  \\
& \textbf{Ours} & \textbf{0.76} & \textbf{19.87} & \textbf{0.078} & \textbf{0.86} & \textbf{27.24} & \textbf{0.023} \\
\bottomrule
\end{tabular}
\vspace{-5mm}
\end{table}

\paragraphcustom{Human Preference Study.} To complement automated evaluation, we conducted a blind human preference study to assess perceived realism and temporal consistency. Raters were presented with side-by-side video pairs from our model and the baseline, alongside the ground-truth video as a reference. They were tasked with selecting the rollout that appeared most realistic and remained most consistent with the ground-truth dynamics.
Our model significantly outperforms the baseline, achieving an $80\%$ preference rate ($174$ wins vs. $43$).This substantial margin is further reflected
 in the Elo ratings, where our model reaches $884.8$ compared to the baseline's $715.2$. 
 
\begin{wrapfigure}{r}{0.45\textwidth}
  \centering
  \vspace{-20pt}
  \includegraphics[width=0.45\textwidth]{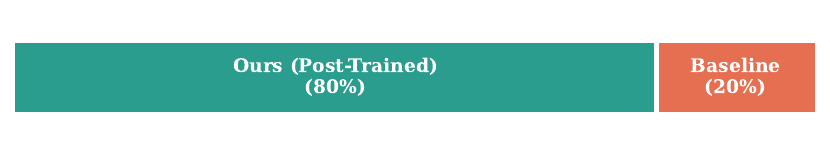}
  \vspace{-25pt}
  \caption{Human preference results.}
  \label{fig:human-pref}
  \vspace{-20pt}
\end{wrapfigure}
 These results confirm that our quantitative gains translate to a qualitatively superior experience, with human observers consistently favoring our model's ability to maintain coherent dynamics over long horizons when compared directly against the ground-truth reference.

\paragraphcustom{Generalization beyond DROID.} 
We finetuned the SVD backbone on Bridge (WidowX)~\cite{ebert2021bridge} and RT-1 (Google Robot)~\cite{rt1} for 100k steps, then applied PersistWorld post-training. PersistWorld post-training improves all three metrics on both embodiments over the SVD backbone. On Bridge, SSIM $0.864\rightarrow0.876$, PSNR $25.20\rightarrow26.36$, LPIPS $0.1061\rightarrow0.0957$; on RT-1, SSIM $0.833\rightarrow0.865$, PSNR $23.81\rightarrow26.63$, LPIPS $0.1683\rightarrow0.1305$.

\paragraphcustom{Policy Evaluation.}
Along with the visual metrics, we also evaluate the correlation between the world-model-based policy evaluations and the real-world policy evaluations. Across three tasks and three policies we find that PersistWorld has better correlation ($0.822$ vs $0.796$) and MMRV ($0.006$ vs $0.053$) than Ctrl-World. 
See Appendix~\ref{app:wm-policy-eval} for more details.

\paragraphcustom{Policy Improvement.}
We ran a Ctrl-World-style policy improvement experiment: $\pi_0$ is full-finetuned on 50 rollouts for 4 block spatial picking task from each WM, then evaluated on a real-robot (5 trials/block). Prompt: ``Pick up the top-left/top-right/bottom-left/bottom-right block''.
Average task progression: base 0.350, +Ctrl-World 0.558, +\textbf{Ours 0.667} (\textbf{$\mathbf{+19\%}$ over Ctrl-World}).

\section{Conclusion}

In this paper, we addressed the critical challenge of exposure bias in action-conditioned robot world models. While existing diffusion-based world models produce high-fidelity short-term clips, their utility as simulators has been limited by compounding errors during autoregressive deployment. We introduced a reinforcement learning post-training framework that bridges this ``closed-loop gap" by training the model on its own generated rollouts rather than ground-truth teacher forcing.

Our technical contributions—adapting a contrastive online reinforcement learning objective to $x_0$-prediction backbones and designing a variable-length branching training protocol with multi-view perceptual rewards—enable the model to remain stable over longer horizons. Our approach establishes a new state-of-the-art on the DROID dataset, significantly reducing perceptual drift and maintaining structural integrity. The 80\% preference rate in our human study and the marked improvement in object and robot-centric metrics suggest that RL post-training is a powerful tool for transforming video generators into reliable, persistent robot simulators. By stabilizing multi-step rollouts, this work paves the way for using world models as scalable, high-fidelity virtual environments for the evaluation and improvement of general-purpose robotic policies.

\paragraphcustom{Limitations and Future Work.} Despite these gains, our approach has limitations. The group-relative training protocol requires sampling
$K=16$ 
independent candidates per update, which increases computational overhead during post-training compared to standard supervised fine-tuning. However, a significant advantage of our framework is its modularity; the contrastive objective is reward-agnostic, meaning the model can be optimized against any combination of perceptual, physical, or task-specific signals. While our current implementation utilizes visual fidelity rewards (LPIPS, SSIM, PSNR) to stabilize the rollouts, these do not yet explicitly enforce physical or geometrical constraints.
Future work will leverage this flexibility to explore more consistency rewards—including physics-informed constraints and geometry-aware metrics—to further enhance physical realism. Additionally, we intend to investigate the application of these persistent world models directly within the policy optimization loop.


\section*{Acknowledgements}
This work was supported by the European Union's Horizon Europe projects AGIMUS (No. 101070165), euROBIN (No. 101070596), ERC FRONTIER (No. 101097822), ELIAS (No. 101120237) and ELLIOT (No. 101214398). Compute resources and infrastructure were supported by the Ministry of Education, Youth and Sports of the Czech Republic through the e-INFRA CZ (ID:90254) and by the European Union's Horizon Europe project CLARA (No. 101136607).

\bibliographystyle{splncs04}
\bibliography{main}

\clearpage
\appendix
\renewcommand{\theHsection}{appendix.\arabic{section}}
\thispagestyle{empty}

\noindent{\LARGE\textbf{Appendix}}
\vspace{3em}

This appendix provides supplementary material organized into five sections. \Cref{app:diffusionnft-x0} gives a complete, self-contained derivation of the post-training loss used to train our model, including a formal theoretical analysis showing why minimizing the loss steers the model toward high-reward outputs. \Cref{app:ablations} presents ablation studies that isolate the contribution of each key design choice, evaluated on the validation split. \Cref{app:wm-policy-eval} examines the use of the world model as a policy evaluation tool, measuring task progression rates across three manipulation tasks. \Cref{app:human-preference-study-details} details the human preference study — participant qualifications, the two-alternative forced-choice (2AFC) interface, and the ELO-based ranking protocol used to aggregate votes. Finally, \Cref{app:algorithm} provides the additional details and full pseudocode for the RL post-training procedure.

\section{Derivation of the Post-Training Objective}
\label{app:diffusionnft-x0}

This appendix provides a complete, self-contained derivation of the post-training objective described in section~\ref{sec:rl} in the main paper. We re-derive the objective from scratch in the $x_0$-prediction parameterization used by our model.

\subsection*{Notation}

We use the following notation throughout this appendix.
\begin{itemize}
  \item $\mathbf{x}_0 \in \mathbb{R}^d$: a clean (noiseless) latent video clip — the quantity the model is trained to predict.
  \item $\mathbf{c} \in \mathbb{R}^{d_c}$: the conditioning signal, comprising the robot action sequence together with any visual or text context provided to the model.
  \item $\sigma > 0$: the noise level. Following the EDM convention~\cite{karras2022elucidatingdesignspacediffusionbased}, a noisy observation is drawn as $\mathbf{x}_\sigma = \mathbf{x}_0 + \sigma\boldsymbol{\varepsilon}$ where $\boldsymbol{\varepsilon} \sim \mathcal{N}(\mathbf{0},\mathbf{I})$ is isotropic Gaussian noise.
  \item $\hat{\mathbf{x}}_{0,\theta} \equiv \hat{\mathbf{x}}_{0,\theta}(\mathbf{x}_\sigma, \sigma, \mathbf{c}) \in \mathbb{R}^d$: the current model's estimate of the clean latent $\mathbf{x}_0$, given noisy input $\mathbf{x}_\sigma$, noise level $\sigma$, and conditioning $\mathbf{c}$.
  \item $\hat{\mathbf{x}}_0^{\mathrm{old}} \equiv \hat{\mathbf{x}}_0^{\mathrm{old}}(\mathbf{x}_\sigma, \sigma, \mathbf{c}) \in \mathbb{R}^d$: the prediction of the \emph{frozen reference model} — a copy of the weights fixed at the start of post-training that is never updated, representing the pre-trained baseline.
  \item $r(\mathbf{x}_0, \mathbf{c}) \in [0,1]$: the normalized reward weight for a generated sample $\mathbf{x}_0$ given conditioning $\mathbf{c}$. It is computed from visual quality metrics within a group of candidates; $r{=}1$ is the best in the group, $r{=}0$ is the worst (see Sec.~\ref{sec:rewards}).
\end{itemize}

\subsection{EDM Preconditioning and the Post-Training Loss}
\label{app:diffusionnft-x0-branches}

\paragraph{EDM preconditioning.}
The world model backbone outputs a corrective term $\mathbf{m}_\theta(\mathbf{x}_\sigma, \sigma, \mathbf{c}) \in \mathbb{R}^d$, which is converted to a clean-latent estimate via the affine EDM preconditioning~\cite{karras2022elucidatingdesignspacediffusionbased}:
\begin{equation}
  \hat{\mathbf{x}}_{0,\theta}
  =
  c_{\mathrm{out}}(\sigma)\,\mathbf{m}_\theta(\mathbf{x}_\sigma, \sigma, \mathbf{c})
  +
  c_{\mathrm{skip}}(\sigma)\,\mathbf{x}_\sigma,
  \label{eq:app-edm-precon}
\end{equation}
where $c_{\mathrm{out}}(\sigma){=}{-}\sigma/\!\sqrt{\sigma^2{+}1}$ and $c_{\mathrm{skip}}(\sigma){=}1/(\sigma^2{+}1)$ are scalar functions of the noise level only. The $c_{\mathrm{skip}}$ term adds back a residual of the noisy input $\mathbf{x}_\sigma$; $c_{\mathrm{out}}$ scales the network's corrective output $\mathbf{m}_\theta$. The identical formula applies to the frozen reference model:
\begin{equation}
  \hat{\mathbf{x}}_0^{\mathrm{old}} = c_{\mathrm{out}}(\sigma)\,\mathbf{m}^{\mathrm{old}}(\mathbf{x}_\sigma, \sigma, \mathbf{c}) + c_{\mathrm{skip}}(\sigma)\,\mathbf{x}_\sigma.
  \label{eq:app-edm-precon-ref}
\end{equation}

\paragraph{Positive and negative branches.}
Given the current-model prediction $\hat{\mathbf{x}}_{0,\theta}$ and the reference prediction $\hat{\mathbf{x}}_0^{\mathrm{old}}$, we construct two branch predictions:
\begin{align}
  \hat{\mathbf{x}}_{0,\theta}^+ &:= (1-\beta)\,\hat{\mathbf{x}}_0^{\mathrm{old}} + \beta\,\hat{\mathbf{x}}_{0,\theta},
  \label{eq:app-pos-branch}\\
  \hat{\mathbf{x}}_{0,\theta}^- &:= (1+\beta)\,\hat{\mathbf{x}}_0^{\mathrm{old}} - \beta\,\hat{\mathbf{x}}_{0,\theta}.
  \label{eq:app-neg-branch}
\end{align}
The \textbf{positive branch} $\hat{\mathbf{x}}_{0,\theta}^+$ interpolates from the reference toward the current model: at $\beta{=}0$ it equals the reference exactly; as $\beta$ grows it moves toward the current model's own prediction. The \textbf{negative branch} $\hat{\mathbf{x}}_{0,\theta}^-$ is its mirror image: it moves \emph{away} from the current model in the same direction the current model has moved from the reference. Together, the two branches bracket the reference prediction symmetrically:
\begin{equation}
  \hat{\mathbf{x}}_{0,\theta}^+ + \hat{\mathbf{x}}_{0,\theta}^- = 2\,\hat{\mathbf{x}}_0^{\mathrm{old}},
  \qquad
  \hat{\mathbf{x}}_{0,\theta}^+ - \hat{\mathbf{x}}_{0,\theta}^- = 2\beta\,\underbrace{(\hat{\mathbf{x}}_{0,\theta} - \hat{\mathbf{x}}_0^{\mathrm{old}})}_{\text{current model's drift from reference}}.
  \label{eq:app-branch-symmetry}
\end{equation}


\paragraph{The post-training loss.}
The post-training objective weights the squared reconstruction errors of the two branches
by the reward weight $r$:
\begin{equation}
  \mathcal{L}(\theta) = \mathbb{E}\Bigl[r\,\|\hat{\mathbf{x}}_{0,\theta}^+ - \mathbf{x}_0\|_2^2
  + (1-r)\,\|\hat{\mathbf{x}}_{0,\theta}^- - \mathbf{x}_0\|_2^2\Bigr],
  \label{eq:app-loss}
\end{equation}
where the expectation is over $(\mathbf{c}, \sigma, \mathbf{x}_\sigma, \mathbf{x}_0)$ drawn jointly. This is identical to Eq.~\eqref{eq:diffusionnft-x0-loss} in the main text. The intuition is direct: for a \emph{high}-reward sample ($r \approx 1$), we minimize the error of the positive branch, which points in the direction the current model has drifted from the reference — thereby \emph{reinforcing} that drift direction. For a \emph{low}-reward sample ($r \approx 0$), we minimize the error of the negative branch, which points in the opposite direction — thereby \emph{penalizing and reversing} that drift. This is the core mechanism by which the loss steers the model toward high-reward outputs.

\subsection{Theoretical Analysis: Why This Loss Improves the Model}
\label{app:diffusionnft-x0-theorems}

We now prove that, under standard assumptions, the unique minimizer of $\mathcal{L}(\theta)$ is a model that has moved precisely in the direction of high-reward samples relative to the reference. The argument proceeds in four steps: (1) decompose the reference distribution into a high-reward part and a low-reward part; (2) show this decomposition lifts to the posterior over clean latents given a noisy observation; (3) identify the reward-aligned direction; (4) show the loss collapses to a single squared error pointing in that direction.

\paragraph{Step 1: Decomposing the reference distribution.}

Let $\pi^{\mathrm{old}}(\mathbf{x}_0|\mathbf{c})$ denote the distribution over clean latents generated by the frozen reference model given conditioning $\mathbf{c}$. We model the reward weight as the conditional probability that a sample $\mathbf{x}_0$ is ``optimal'' given $\mathbf{c}$: introducing a latent binary optimality label $o \in \{0,1\}$, we set
\begin{equation}
  r(\mathbf{x}_0, \mathbf{c}) := P(o=1 \mid \mathbf{x}_0, \mathbf{c}) \in [0,1].
\end{equation}
Define the \emph{positive distribution} $\pi^+$ (samples conditioned on being optimal) and \emph{negative distribution} $\pi^-$ (samples conditioned on being suboptimal) via Bayes' rule:
\begin{align}
  \pi^{+}(\mathbf{x}_0|\mathbf{c})
  &:= P(\mathbf{x}_0 \mid o{=}1, \mathbf{c})
  = \frac{r(\mathbf{x}_0,\mathbf{c})}{Z(\mathbf{c})}\,\pi^{\mathrm{old}}(\mathbf{x}_0|\mathbf{c}),
  \label{eq:app-pi-plus}\\
  \pi^{-}(\mathbf{x}_0|\mathbf{c})
  &:= P(\mathbf{x}_0 \mid o{=}0, \mathbf{c})
  = \frac{1 - r(\mathbf{x}_0,\mathbf{c})}{1 - Z(\mathbf{c})}\,\pi^{\mathrm{old}}(\mathbf{x}_0|\mathbf{c}),
  \label{eq:app-pi-minus}
\end{align}
where
\begin{equation}
  Z(\mathbf{c}) := \mathbb{E}_{\pi^{\mathrm{old}}(\mathbf{x}_0|\mathbf{c})}[r(\mathbf{x}_0,\mathbf{c})] \in (0,1)
  \label{eq:app-Z}
\end{equation}
is the \emph{partition function} — the expected reward under the reference model, which normalizes $\pi^+$ and $\pi^-$ to be valid probability distributions. By the law of total probability, the reference distribution is a weighted mixture of its two parts:
\begin{equation}
  \pi^{\mathrm{old}}(\mathbf{x}_0|\mathbf{c})
  = Z(\mathbf{c})\,\pi^{+}(\mathbf{x}_0|\mathbf{c}) + (1-Z(\mathbf{c}))\,\pi^{-}(\mathbf{x}_0|\mathbf{c}).
  \label{eq:app-prior-mixture}
\end{equation}
\emph{Intuition.} The reference model generates samples from $\pi^{\mathrm{old}}$; a fraction $Z(\mathbf{c})$ of those happen to be high-quality (distributed as $\pi^+$, up-weighted by their reward $r$) and the rest are low-quality (distributed as $\pi^-$, up-weighted by $1-r$). The reward simply measures the relative density of $\pi^+$ with respect to $\pi^{\mathrm{old}}$.

\paragraph{Step 2: Lifting the decomposition to noisy observations.}

The forward noising kernel is $q_\sigma(\mathbf{x}_\sigma|\mathbf{x}_0) = \mathcal{N}(\mathbf{x}_\sigma;\mathbf{x}_0,\sigma^2\mathbf{I})$. Let $\pi^\star(\mathbf{x}_\sigma|\mathbf{c}) := \int q_\sigma(\mathbf{x}_\sigma|\mathbf{x}_0)\,\pi^\star(\mathbf{x}_0|\mathbf{c})\,\mathrm{d}\mathbf{x}_0$ be the marginal distribution of the noisy observation under $\star \in \{\mathrm{old},+,-\}$, and let $\pi^\star(\mathbf{x}_0|\mathbf{x}_\sigma,\mathbf{c})$ denote the corresponding posterior over clean latents given a noisy observation $\mathbf{x}_\sigma$.

\begin{lemma}[Posterior mixture]
\label{lem:posterior-mixture}
The posterior $\pi^{\mathrm{old}}(\mathbf{x}_0|\mathbf{x}_\sigma,\mathbf{c})$ is a mixture of the positive and negative posteriors:
\begin{equation}
  \pi^{\mathrm{old}}(\mathbf{x}_0|\mathbf{x}_\sigma,\mathbf{c})
  =
  \alpha(\mathbf{x}_\sigma,\mathbf{c})\,\pi^{+}(\mathbf{x}_0|\mathbf{x}_\sigma,\mathbf{c})
  +
  \bigl(1-\alpha(\mathbf{x}_\sigma,\mathbf{c})\bigr)\,\pi^{-}(\mathbf{x}_0|\mathbf{x}_\sigma,\mathbf{c}),
  \label{eq:app-posterior-mixture}
\end{equation}
where the data-dependent mixing weight (abbreviated $\alpha \equiv \alpha(\mathbf{x}_\sigma,\mathbf{c})$ below) is
\begin{equation}
  \alpha(\mathbf{x}_\sigma,\mathbf{c})
  :=
  \frac{Z(\mathbf{c})\,\pi^{+}(\mathbf{x}_\sigma|\mathbf{c})}{\pi^{\mathrm{old}}(\mathbf{x}_\sigma|\mathbf{c})}
  \in [0,1].
  \label{eq:app-alpha}
\end{equation}
\end{lemma}

\begin{proof}
Marginalise the prior mixture (Eq.~\eqref{eq:app-prior-mixture}) through $q_\sigma$. Since $q_\sigma$ does not depend on the optimality label $o$, it acts identically on each mixture component:
\begin{equation}
  \pi^{\mathrm{old}}(\mathbf{x}_\sigma|\mathbf{c}) = Z(\mathbf{c})\,\pi^+(\mathbf{x}_\sigma|\mathbf{c}) + (1-Z(\mathbf{c}))\,\pi^-(\mathbf{x}_\sigma|\mathbf{c}).
  \label{eq:app-marginal-mixture}
\end{equation}
Apply Bayes' rule to $\pi^{\mathrm{old}}(\mathbf{x}_0|\mathbf{x}_\sigma,\mathbf{c})$ and substitute Eq.~\eqref{eq:app-prior-mixture}:
\begin{align}
  \pi^{\mathrm{old}}(\mathbf{x}_0|\mathbf{x}_\sigma,\mathbf{c})
  &= \frac{q_\sigma(\mathbf{x}_\sigma|\mathbf{x}_0)\,\pi^{\mathrm{old}}(\mathbf{x}_0|\mathbf{c})}{\pi^{\mathrm{old}}(\mathbf{x}_\sigma|\mathbf{c})} \nonumber\\
  &= \frac{q_\sigma(\mathbf{x}_\sigma|\mathbf{x}_0)}{\pi^{\mathrm{old}}(\mathbf{x}_\sigma|\mathbf{c})}
    \Bigl[Z\,\pi^+(\mathbf{x}_0|\mathbf{c}) + (1{-}Z)\,\pi^-(\mathbf{x}_0|\mathbf{c})\Bigr].
\end{align}
For each mixture component, Bayes' rule gives
\begin{equation*}
  q_\sigma(\mathbf{x}_\sigma|\mathbf{x}_0)\,\pi^\star(\mathbf{x}_0|\mathbf{c})
  = \pi^\star(\mathbf{x}_\sigma|\mathbf{c})\,\pi^\star(\mathbf{x}_0|\mathbf{x}_\sigma,\mathbf{c}).
\end{equation*}
Substituting and grouping terms:
\begin{align}
  \pi^{\mathrm{old}}(\mathbf{x}_0|\mathbf{x}_\sigma,\mathbf{c})
  &= \frac{Z\,\pi^+(\mathbf{x}_\sigma|\mathbf{c})}{\pi^{\mathrm{old}}(\mathbf{x}_\sigma|\mathbf{c})}\,\pi^+(\mathbf{x}_0|\mathbf{x}_\sigma,\mathbf{c})
  + \frac{(1{-}Z)\,\pi^-(\mathbf{x}_\sigma|\mathbf{c})}{\pi^{\mathrm{old}}(\mathbf{x}_\sigma|\mathbf{c})}\,\pi^-(\mathbf{x}_0|\mathbf{x}_\sigma,\mathbf{c}) \nonumber\\
  &= \alpha\,\pi^+(\mathbf{x}_0|\mathbf{x}_\sigma,\mathbf{c})
  + (1-\alpha)\,\pi^-(\mathbf{x}_0|\mathbf{x}_\sigma,\mathbf{c}),
\end{align}
The first coefficient is exactly $\alpha(\mathbf{x}_\sigma,\mathbf{c})$ as defined in Eq.~\eqref{eq:app-alpha}.
The second equals $1-\alpha$: dividing Eq.~\eqref{eq:app-marginal-mixture} by $\pi^{\mathrm{old}}(\mathbf{x}_\sigma|\mathbf{c})$ gives
$\alpha + (1{-}Z)\pi^-(\mathbf{x}_\sigma|\mathbf{c})/\pi^{\mathrm{old}}(\mathbf{x}_\sigma|\mathbf{c}) = 1$,
so the second weight is indeed $1-\alpha$.
Non-negativity of $\pi^\pm$ and $\pi^{\mathrm{old}} \geq Z\pi^+$ ensure $\alpha \in [0,1]$.
\end{proof}

\emph{Intuition.} Given a noisy observation $\mathbf{x}_\sigma$, our uncertainty about the underlying clean latent $\mathbf{x}_0$ splits between two competing hypotheses: it came from the high-reward regime ($\pi^+$, posterior weight $\alpha$) or the low-reward regime ($\pi^-$, weight $1-\alpha$). The mixing weight $\alpha(\mathbf{x}_\sigma,\mathbf{c})$ is the posterior probability that $\mathbf{x}_\sigma$ originated from a high-reward sample; it is large when $\mathbf{x}_\sigma$ is more likely under $\pi^+$ than under $\pi^{\mathrm{old}}$.

\paragraph{Step 3: Identifying the reward-aligned improvement direction.}

For $\star \in \{\mathrm{old}, +, -\}$, define the posterior mean — the expected clean latent given $\mathbf{x}_\sigma$ under distribution $\pi^\star$:
\begin{equation}
  \boldsymbol{\mu}^{\star} \equiv \boldsymbol{\mu}^{\star}(\mathbf{x}_\sigma,\mathbf{c},\sigma)
  :=
  \mathbb{E}_{\pi^\star(\mathbf{x}_0|\mathbf{x}_\sigma,\mathbf{c})}[\mathbf{x}_0],
  \qquad
  \star\in\{\mathrm{old},+,-\}.
  \label{eq:app-mu}
\end{equation}
Thus $\boldsymbol{\mu}^{\mathrm{old}}$ is the expected clean latent under the reference model, $\boldsymbol{\mu}^+$ under the high-reward subset, and $\boldsymbol{\mu}^-$ under the low-reward subset. Taking expectations of both sides of Eq.~\eqref{eq:app-posterior-mixture}:
\begin{equation}
  \boldsymbol{\mu}^{\mathrm{old}} = \alpha\,\boldsymbol{\mu}^{+} + (1-\alpha)\,\boldsymbol{\mu}^{-}.
  \label{eq:app-mean-mixture}
\end{equation}
Rearranging, the displacement from $\boldsymbol{\mu}^{\mathrm{old}}$ to $\boldsymbol{\mu}^+$ and from $\boldsymbol{\mu}^-$ to $\boldsymbol{\mu}^{\mathrm{old}}$ are parallel and proportional, defining a single \emph{improvement direction}:
\begin{equation}
  \Delta_{x_0}
  :=
  (1-\alpha)\,(\boldsymbol{\mu}^{\mathrm{old}}-\boldsymbol{\mu}^{-})
  =
  \alpha\,(\boldsymbol{\mu}^{+}-\boldsymbol{\mu}^{\mathrm{old}}).
  \label{eq:app-delta}
\end{equation}
The vector $\Delta_{x_0}$ points from the low-reward mean $\boldsymbol{\mu}^-$ toward the high-reward mean $\boldsymbol{\mu}^+$, with $\boldsymbol{\mu}^{\mathrm{old}}$ lying on the segment between them. From Eq.~\eqref{eq:app-delta} we also read off the inverse relations:
\begin{equation}
  \boldsymbol{\mu}^{+} = \boldsymbol{\mu}^{\mathrm{old}} + \frac{\Delta_{x_0}}{\alpha},
  \qquad
  \boldsymbol{\mu}^{-} = \boldsymbol{\mu}^{\mathrm{old}} - \frac{\Delta_{x_0}}{1-\alpha}.
  \label{eq:app-mu-pm}
\end{equation}

\paragraph{Step 4: The loss collapses to a squared error in the improvement direction.}

\begin{theorem}[Optimal predictor]
\label{thm:optimal-predictor}
Under unlimited data and model capacity, the unique minimizer of $\mathcal{L}(\theta)$ is the predictor that moves from the reference posterior mean by exactly $2/\beta$ steps in the improvement direction:
\begin{equation}
  \hat{\mathbf{x}}_{0,\theta^*}(\mathbf{x}_\sigma,\mathbf{c},\sigma)
  =
  \boldsymbol{\mu}^{\mathrm{old}}(\mathbf{x}_\sigma,\mathbf{c},\sigma)
  +
  \frac{2}{\beta}\,\Delta_{x_0}(\mathbf{x}_\sigma,\mathbf{c},\sigma).
  \label{eq:app-optimal}
\end{equation}
\end{theorem}

\begin{proof}
We introduce the shorthand
\begin{equation}
  \mathbf{d} \;:=\; \hat{\mathbf{x}}_{0,\theta} - \boldsymbol{\mu}^{\mathrm{old}},
  \label{eq:app-d-def}
\end{equation}
the current model's deviation from the reference posterior mean $\boldsymbol{\mu}^{\mathrm{old}}$. Every $\theta$-dependent quantity in $\mathcal{L}(\theta)$ can be written in terms of $\mathbf{d}$; the goal is to show the loss is a pure squared error in $\mathbf{d}$ centered at the improvement direction $2\Delta_{x_0}/\beta$.
\vspace{1pt}

\textbf{Algebraic step 1: Rewrite the loss using the distributional decomposition.}
From Eq.~\eqref{eq:app-pi-plus}, $r(\mathbf{x}_0,\mathbf{c})\,\pi^{\mathrm{old}}(\mathbf{x}_0|\mathbf{c}) = Z(\mathbf{c})\,\pi^+(\mathbf{x}_0|\mathbf{c})$. Substituting into the inner expectation over $\pi^{\mathrm{old}}(\mathbf{x}_0|\mathbf{x}_\sigma,\mathbf{c})$ and applying Bayes' rule gives:
\begin{equation}
  r(\mathbf{x}_0,\mathbf{c})\,\pi^{\mathrm{old}}(\mathbf{x}_0|\mathbf{x}_\sigma,\mathbf{c})
  = \alpha(\mathbf{x}_\sigma,\mathbf{c})\,\pi^{+}(\mathbf{x}_0|\mathbf{x}_\sigma,\mathbf{c}),
  \label{eq:app-r-pi}
\end{equation}
and analogously $(1{-}r)\,\pi^{\mathrm{old}}(\mathbf{x}_0|\mathbf{x}_\sigma,\mathbf{c}) = (1{-}\alpha)\,\pi^{-}(\mathbf{x}_0|\mathbf{x}_\sigma,\mathbf{c})$. Using Eqs.~\eqref{eq:app-r-pi} to rewrite the reward-weighted inner expectation in $\mathcal{L}(\theta)$:
\begin{align}
  \mathcal{L}(\theta)
  &=
  \mathbb{E}_{\mathbf{c},\sigma,\mathbf{x}_\sigma}
  \Bigl[
    \alpha\,\mathbb{E}_{\pi^{+}(\mathbf{x}_0|\mathbf{x}_\sigma,\mathbf{c})}
    \|\hat{\mathbf{x}}_{0,\theta}^+ - \mathbf{x}_0\|_2^2 \nonumber\\
  &\hspace{6em}
    +
    (1-\alpha)\,\mathbb{E}_{\pi^{-}(\mathbf{x}_0|\mathbf{x}_\sigma,\mathbf{c})}
    \|\hat{\mathbf{x}}_{0,\theta}^- - \mathbf{x}_0\|_2^2
  \Bigr].
  \label{eq:diffusionnft-x0-theorem-step1}
\end{align}

\textbf{Algebraic step 2: Replace the inner expectation with the posterior mean.}
For any fixed vector $\hat{\mathbf{f}} \in \mathbb{R}^d$ and any distribution $p(\mathbf{y})$, the variance decomposition gives $\mathbb{E}_p\|\hat{\mathbf{f}} - \mathbf{y}\|_2^2 = \|\hat{\mathbf{f}} - \mathbb{E}_p[\mathbf{y}]\|_2^2 + \mathrm{Var}_p[\mathbf{y}]$, so the squared error is minimized over $\hat{\mathbf{f}}$ by the mean $\mathbb{E}_p[\mathbf{y}]$, with the residual variance $\mathrm{Var}_p[\mathbf{y}]$ being constant in $\hat{\mathbf{f}}$. Applying this to each term in Eq.~\eqref{eq:diffusionnft-x0-theorem-step1} — the first has prediction $\hat{\mathbf{x}}_{0,\theta}^+$ fixed with respect to $\mathbf{x}_0$, and posterior mean $\boldsymbol{\mu}^+$; the second has $\hat{\mathbf{x}}_{0,\theta}^-$ and posterior mean $\boldsymbol{\mu}^-$:
\begin{align}
  \mathcal{L}(\theta)
  &=
  \mathbb{E}_{\mathbf{c},\sigma,\mathbf{x}_\sigma}
  \Bigl[
    \alpha\|\hat{\mathbf{x}}_{0,\theta}^+ - \boldsymbol{\mu}^{+}\|_2^2
    +
    (1-\alpha)\|\hat{\mathbf{x}}_{0,\theta}^- - \boldsymbol{\mu}^{-}\|_2^2
  \Bigr]
  + C, \label{eq:diffusionnft-x0-theorem-step2}
\end{align}
where $C = \mathbb{E}[\alpha\,\mathrm{Var}_{\pi^+}[\mathbf{x}_0|\mathbf{x}_\sigma,\mathbf{c}] + (1{-}\alpha)\,\mathrm{Var}_{\pi^-}[\mathbf{x}_0|\mathbf{x}_\sigma,\mathbf{c}]\,]$
does not depend on $\theta$.

\textbf{Algebraic step 3: Expand in terms of $\mathbf{d}$.}
Step~2 showed that the MSE $\mathbb{E}\|\hat{f} - \mathbf{x}_0\|_2^2$ is minimized by the posterior mean. Applying this to the reference model's own pre-training loss (i.e.\ it was trained to minimize MSE against samples from $\pi^{\mathrm{old}}$), an unlimited-capacity reference model converges to the posterior mean $\boldsymbol{\mu}^{\mathrm{old}}$, so $\hat{\mathbf{x}}_0^{\mathrm{old}} = \boldsymbol{\mu}^{\mathrm{old}}$.
Meanwhile, the current model satisfies $\hat{\mathbf{x}}_{0,\theta} = \boldsymbol{\mu}^{\mathrm{old}} + \mathbf{d}$ by definition of $\mathbf{d}$ (Eq.~\eqref{eq:app-d-def}). Substituting both into Eqs.~\eqref{eq:app-pos-branch}--\eqref{eq:app-neg-branch} and subtracting $\boldsymbol{\mu}^\pm$ via Eq.~\eqref{eq:app-mu-pm}:
\begin{align}
  \hat{\mathbf{x}}_{0,\theta}^+ - \boldsymbol{\mu}^+
  &= \bigl[(1{-}\beta)\boldsymbol{\mu}^{\mathrm{old}} + \beta(\boldsymbol{\mu}^{\mathrm{old}} {+} \mathbf{d})\bigr]
    - \bigl[\boldsymbol{\mu}^{\mathrm{old}} + \tfrac{\Delta_{x_0}}{\alpha}\bigr]
  = \beta\mathbf{d} - \frac{\Delta_{x_0}}{\alpha},
  \label{eq:app-resid-plus}\\
  \hat{\mathbf{x}}_{0,\theta}^- - \boldsymbol{\mu}^-
  &= \bigl[(1{+}\beta)\boldsymbol{\mu}^{\mathrm{old}} - \beta(\boldsymbol{\mu}^{\mathrm{old}} {+} \mathbf{d})\bigr]
    - \bigl[\boldsymbol{\mu}^{\mathrm{old}} - \tfrac{\Delta_{x_0}}{1-\alpha}\bigr]
  = -\beta\mathbf{d} + \frac{\Delta_{x_0}}{1-\alpha}.
  \label{eq:app-resid-minus}
\end{align}
Substituting Eqs.~\eqref{eq:app-resid-plus}--\eqref{eq:app-resid-minus} into Eq.~\eqref{eq:diffusionnft-x0-theorem-step2} and expanding:
\begin{align}
  &\alpha\|\hat{\mathbf{x}}_{0,\theta}^+ - \boldsymbol{\mu}^+\|_2^2
  + (1-\alpha)\|\hat{\mathbf{x}}_{0,\theta}^- - \boldsymbol{\mu}^-\|_2^2 \nonumber\\
  &=
  \alpha\!\left[\beta^2\|\mathbf{d}\|_2^2 - \tfrac{2\beta}{\alpha}\mathbf{d}^\top\Delta_{x_0} + \tfrac{\|\Delta_{x_0}\|_2^2}{\alpha^2}\right]
  + (1{-}\alpha)\!\left[\beta^2\|\mathbf{d}\|_2^2 - \tfrac{2\beta}{1-\alpha}\mathbf{d}^\top\Delta_{x_0} + \tfrac{\|\Delta_{x_0}\|_2^2}{(1-\alpha)^2}\right] \nonumber\\
  &= \beta^2\|\mathbf{d}\|_2^2 - 4\beta\,\mathbf{d}^\top\Delta_{x_0} + \frac{\|\Delta_{x_0}\|_2^2}{\alpha(1-\alpha)},
  \label{eq:diffusionnft-x0-theorem-step3}
\end{align}
collecting: the $\|\mathbf{d}\|_2^2$ terms give $(\alpha{+}1{-}\alpha)\beta^2\|\mathbf{d}\|_2^2 = \beta^2\|\mathbf{d}\|_2^2$; the cross-terms give $({-}2\beta{-}2\beta)\mathbf{d}^\top\Delta_{x_0}$; the $\|\Delta_{x_0}\|_2^2$ terms give $1/\alpha + 1/(1{-}\alpha) = 1/(\alpha(1{-}\alpha))$.

\textbf{Algebraic step 4: Complete the square in $\mathbf{d}$.}
\begin{equation}
  \beta^2\|\mathbf{d}\|_2^2 - 4\beta\,\mathbf{d}^\top\Delta_{x_0}
  =
  \beta^2\!\left\|\mathbf{d} - \frac{2\Delta_{x_0}}{\beta}\right\|_2^2 - 4\|\Delta_{x_0}\|_2^2.
\end{equation}
Substituting into Eq.~\eqref{eq:diffusionnft-x0-theorem-step3} and then into Eq.~\eqref{eq:diffusionnft-x0-theorem-step2}:
\begin{equation}
  \mathcal{L}(\theta)
  =
  \beta^2\,\mathbb{E}_{\mathbf{c},\sigma,\mathbf{x}_\sigma}
  \left\|
    \hat{\mathbf{x}}_{0,\theta}
    -
    \left(\boldsymbol{\mu}^{\mathrm{old}}+\frac{2}{\beta}\Delta_{x_0}\right)
  \right\|_2^2
  + C',
\end{equation}
where $C' = C + \mathbb{E}[\|\Delta_{x_0}\|_2^2/(\alpha(1{-}\alpha)) - 4\|\Delta_{x_0}\|_2^2]$ is independent of $\theta$. Since $\beta^2 > 0$, the unique minimizer sets $\hat{\mathbf{x}}_{0,\theta} = \boldsymbol{\mu}^{\mathrm{old}} + (2/\beta)\Delta_{x_0}$ pointwise, giving Eq.~\eqref{eq:app-optimal}.
\end{proof}

The optimal model is the reference posterior mean $\boldsymbol{\mu}^{\mathrm{old}}$ shifted by $(2/\beta)$ steps in the reward-aligned direction $\Delta_{x_0}$, which by Eq.~\eqref{eq:app-delta} points from $\boldsymbol{\mu}^-$ (low-reward samples) toward $\boldsymbol{\mu}^+$ (high-reward samples). The hyperparameter $\beta$ controls how far the model is shifted: larger $\beta$ produces a stronger per-step signal but also moves the branches farther from the reference, potentially destabilising training.

\paragraph{Connection to other parameterisations.}
For Gaussian noising schedules, velocity predictors $\mathbf{v}_\theta$ used by flow-matching models are related to $x_0$ predictors by a per-noise-level affine transformation that does not depend on $\theta$. Consequently, differences between any two predictors in $x_0$ space are proportional to their differences in velocity space. The improvement direction $\Delta_{x_0}$ and the result of Theorem~\ref{thm:optimal-predictor} therefore translate directly to velocity-prediction parameterisations up to a $\sigma$-dependent scalar, confirming that the theoretical guarantees are parameterisation-agnostic.

\section{Ablation Studies}
\label{app:ablations}

We report ablation results for the key design choices of our post-training. Each table isolates one factor while keeping all other hyperparameters fixed to the default configuration used in the main experiments. Metrics are evaluated on the DROID validation split over $10$-step autoregressive rollouts; \textbf{higher} SSIM and PSNR and \textbf{lower} LPIPS indicate better visual quality. Specifically, we report the performance after the \textbf{same number of gradient updates (training steps) have been performed} on the model. Notably, this is different from other RL literature which report performance after the same number of outer iterations (ignoring the number of gradient updates performed during the inner iterations). The ablations collectively justify design choices made in the main paper.


\begin{table}[ht]
\centering
\caption{Ablation of the post-training learning rate.}
\label{tab:abl-lr}
\vspace{-6pt}
\begin{tabular}{@{}l ccc ccc@{}}
\toprule
\multirow{2}{*}{Learning Rate} &
  \multicolumn{3}{c}{External} & \multicolumn{3}{c}{Wrist} \\
\cmidrule(lr){2-4}\cmidrule(lr){5-7}
& SSIM$\uparrow$ & PSNR$\uparrow$ & LPIPS$\downarrow$ &
  SSIM$\uparrow$ & PSNR$\uparrow$ & LPIPS$\downarrow$ \\
\midrule
$3\times10^{-4}$ & $\underline{0.865}$ & $\underline{25.02}$ & $\underline{0.1230}$ & $\underline{0.706}$ & $\underline{19.95}$ & $\underline{0.3603}$ \\
$1\times10^{-4}$ \textbf{(Ours)} & $\mathbf{0.872}$ & $\mathbf{25.52}$ & $\mathbf{0.1162}$ & $\mathbf{0.721}$ & $\mathbf{20.40}$ & $\mathbf{0.3387}$ \\
\bottomrule
\end{tabular}
\vspace{4pt}
\end{table}


\begin{table}[ht]
\centering
\caption{Ablation of the reward signal used during DiffusionNFT post-training.}
\label{tab:abl-reward}
\vspace{-6pt}
\begin{tabular}{@{}l ccc ccc@{}}
\toprule
\multirow{2}{*}{Reward} &
  \multicolumn{3}{c}{External} & \multicolumn{3}{c}{Wrist} \\
\cmidrule(lr){2-4}\cmidrule(lr){5-7}
& SSIM$\uparrow$ & PSNR$\uparrow$ & LPIPS$\downarrow$ &
  SSIM$\uparrow$ & PSNR$\uparrow$ & LPIPS$\downarrow$ \\
\midrule
LPIPS only & $0.868$ & $25.17$ & $\mathbf{0.1156}$ & $0.708$ & $19.90$ & $\mathbf{0.3242}$ \\
SSIM only  & $\mathbf{0.872}$ & $25.40$ & $0.1178$ & $\mathbf{0.724}$ & $20.24$ & $0.3516$ \\
PSNR only  & $\underline{0.871}$ & $\mathbf{25.57}$ & $0.1189$ & $0.716$ & $\mathbf{20.45}$ & $0.3544$ \\
Combined \textbf{(Ours)}       & $\mathbf{0.872}$ & $\underline{25.52}$ & $\underline{0.1162}$ & $\underline{0.721}$ & $\underline{20.40}$ & $\underline{0.3387}$ \\
\bottomrule
\end{tabular}
\vspace{4pt}
\end{table}


\begin{table}[ht]
\centering
\caption{Ablation of the context-window curriculum. \emph{Curriculum} is the growing schedule from~\protect\cite{wang2026worldcompass}; \emph{Random} samples uniformly from the specified range each step; \emph{Fixed} keeps a constant window size throughout training.}
\label{tab:abl-curriculum}
\vspace{-6pt}
\begin{tabular}{@{}l ccc ccc@{}}
\toprule
\multirow{2}{*}{Strategy} &
  \multicolumn{3}{c}{External} & \multicolumn{3}{c}{Wrist} \\
\cmidrule(lr){2-4}\cmidrule(lr){5-7}
& SSIM$\uparrow$ & PSNR$\uparrow$ & LPIPS$\downarrow$ &
  SSIM$\uparrow$ & PSNR$\uparrow$ & LPIPS$\downarrow$ \\
\midrule
Fixed                        &    &    &    &    &    &    \\
\quad Size 3                 & $\underline{0.872}$ & $\mathbf{25.53}$ & $0.1170$ & $\mathbf{0.723}$ & $\mathbf{20.51}$ & $\mathbf{0.3381}$ \\
\quad Size 6                 & $0.871$ & $25.41$ & $0.1174$ & $0.719$ & $20.33$ & $0.3423$ \\
Curriculum~\cite{wang2026worldcompass}                   & $0.871$ & $25.50$ & $0.1169$ & $0.721$ & $20.43$ & $0.3412$ \\
Random                       &    &    &    &    &    &    \\
\quad Size 0--9 \textbf{(Ours)}             & $\underline{0.872}$ & $\underline{25.52}$ & $\underline{0.1162}$ & $0.721$ & $20.40$ & $\underline{0.3387}$ \\
\quad Size 0--4              & $\mathbf{0.873}$ & $\underline{25.52}$ & $\mathbf{0.1161}$ & $\underline{0.722}$ & $\underline{20.44}$ & $0.3391$ \\
\bottomrule
\end{tabular}
\vspace{4pt}
\end{table}


\begin{table}[ht]
\centering
\caption{Ablation of the GRPO group size and the number of candidates retained via best-of-$N$ filtering (evaluated at group size~16).}
\label{tab:abl-group}
\vspace{-6pt}
\begin{tabular}{@{}l ccc ccc@{}}
\toprule
\multirow{2}{*}{Config} &
  \multicolumn{3}{c}{External} & \multicolumn{3}{c}{Wrist} \\
\cmidrule(lr){2-4}\cmidrule(lr){5-7}
& SSIM$\uparrow$ & PSNR$\uparrow$ & LPIPS$\downarrow$ &
  SSIM$\uparrow$ & PSNR$\uparrow$ & LPIPS$\downarrow$ \\
\midrule
Group Size 4              & $\underline{0.875}$ & $25.62$ & $0.1142$ & $0.724$ & $20.56$ & $0.3379$ \\
Group Size 8              & $0.873$ & $25.61$ & $0.1176$ & $0.721$ & $20.45$ & $0.3471$ \\
Group Size 16             &    &    &    &    &    &    \\
\quad No best of $N$      & $0.874$ & $\underline{25.69}$ & $\underline{0.1127}$ & $\underline{0.725}$ & $\underline{20.64}$ & $\underline{0.3255}$ \\
\quad Best of 5   \textbf{(Ours)}        & $0.872$ & $25.52$ & $0.1162$ & $0.721$ & $20.40$ & $0.3387$ \\
\quad Best of 2           & $\mathbf{0.876}$ & $\mathbf{25.79}$ & $\mathbf{0.1107}$ & $\mathbf{0.727}$ & $\mathbf{20.69}$ & $\mathbf{0.3212}$ \\
\bottomrule
\end{tabular}
\vspace{4pt}
\end{table}


\begin{table}[ht]
\centering
\caption{Ablation of the future prediction horizon $H$ (number of frames generated per autoregressive step) used during post-training.}
\label{tab:abl-horizon}
\vspace{-6pt}
\begin{tabular}{@{}l ccc ccc@{}}
\toprule
\multirow{2}{*}{Horizon $H$} &
  \multicolumn{3}{c}{External} & \multicolumn{3}{c}{Wrist} \\
\cmidrule(lr){2-4}\cmidrule(lr){5-7}
& SSIM$\uparrow$ & PSNR$\uparrow$ & LPIPS$\downarrow$ &
  SSIM$\uparrow$ & PSNR$\uparrow$ & LPIPS$\downarrow$ \\
\midrule
$H=3$ & $\mathbf{0.872}$ & $\underline{25.41}$ & $\underline{0.1169}$ & $\underline{0.719}$ & $\underline{20.31}$ & $\underline{0.3424}$ \\
$H=1$ \textbf{(Ours)} & $\mathbf{0.872}$ & $\mathbf{25.52}$ & $\mathbf{0.1162}$ & $\mathbf{0.721}$ & $\mathbf{20.40}$ & $\mathbf{0.3387}$ \\
\bottomrule
\end{tabular}
\vspace{4pt}
\end{table}


\begin{table}[ht]
\centering
\caption{Ablation of the wrist-camera loss weight $w_{\mathrm{wrist}}$ relative to the external cameras (weight~1).}
\label{tab:abl-view}
\vspace{-6pt}
\begin{tabular}{@{}l ccc ccc@{}}
\toprule
\multirow{2}{*}{$w_{\mathrm{wrist}}$} &
  \multicolumn{3}{c}{External} & \multicolumn{3}{c}{Wrist} \\
\cmidrule(lr){2-4}\cmidrule(lr){5-7}
& SSIM$\uparrow$ & PSNR$\uparrow$ & LPIPS$\downarrow$ &
  SSIM$\uparrow$ & PSNR$\uparrow$ & LPIPS$\downarrow$ \\
\midrule
$w_{\mathrm{wrist}}=2$ & $\mathbf{0.872}$ & $\underline{25.45}$ & $\underline{0.1169}$ & $\mathbf{0.722}$ & $\mathbf{20.47}$ & $\mathbf{0.3357}$ \\
$w_{\mathrm{wrist}}=1$ \textbf{(Ours)} & $\mathbf{0.872}$ & $\mathbf{25.52}$ & $\mathbf{0.1162}$ & $\underline{0.721}$ & $\underline{20.40}$ & $\underline{0.3387}$ \\
\bottomrule
\end{tabular}
\vspace{4pt}
\end{table}


\begin{table}[ht]
\centering
\caption{Effect of KL regularisation applied to the post-training objective.}
\label{tab:abl-kl}
\vspace{-6pt}
\begin{tabular}{@{}l ccc ccc@{}}
\toprule
\multirow{2}{*}{KL Reg.} &
  \multicolumn{3}{c}{External} & \multicolumn{3}{c}{Wrist} \\
\cmidrule(lr){2-4}\cmidrule(lr){5-7}
& SSIM$\uparrow$ & PSNR$\uparrow$ & LPIPS$\downarrow$ &
  SSIM$\uparrow$ & PSNR$\uparrow$ & LPIPS$\downarrow$ \\
\midrule
Without & $\mathbf{0.873}$ & $\underline{25.50}$ & $\underline{0.1163}$ & $\underline{0.720}$ & $\underline{20.38}$ & $\underline{0.3392}$ \\
With \textbf{(Ours)}   & $\underline{0.872}$ & $\mathbf{25.52}$ & $\mathbf{0.1162}$ & $\mathbf{0.721}$ & $\mathbf{20.40}$ & $\mathbf{0.3387}$ \\
\bottomrule
\end{tabular}
\vspace{4pt}
\end{table}


\begin{table}[ht]
\centering
\caption{Effect of applying a warming EMA schedule to the frozen reference (old) policy, which gradually interpolates it toward the learning policy during training.}
\label{tab:abl-ema-old}
\vspace{-6pt}
\begin{tabular}{@{}l ccc ccc@{}}
\toprule
\multirow{2}{*}{Reference Policy EMA} &
  \multicolumn{3}{c}{External} & \multicolumn{3}{c}{Wrist} \\
\cmidrule(lr){2-4}\cmidrule(lr){5-7}
& SSIM$\uparrow$ & PSNR$\uparrow$ & LPIPS$\downarrow$ &
  SSIM$\uparrow$ & PSNR$\uparrow$ & LPIPS$\downarrow$ \\
\midrule
No EMA schedule        & $\mathbf{0.872}$ & $\underline{25.51}$ & $\mathbf{0.1162}$ & $\mathbf{0.722}$ & $\mathbf{20.46}$ & $\underline{0.3390}$ \\
EMA rising to $0.5$ \textbf{(Ours)}   & $\mathbf{0.872}$ & $\mathbf{25.52}$ & $\mathbf{0.1162}$ & $\underline{0.721}$ & $\underline{20.40}$ & $\mathbf{0.3387}$ \\
\bottomrule
\end{tabular}
\vspace{4pt}
\end{table}


\begin{table}[h!]
\centering
\caption{Effect of applying an exponential moving average (EMA) to the learning (fine-tuned) policy weights during post-training.}
\vspace{-6pt}
\label{tab:abl-ema-learn}
\begin{tabular}{@{}l ccc ccc@{}}
\toprule
\multirow{2}{*}{Policy EMA} &
  \multicolumn{3}{c}{External} & \multicolumn{3}{c}{Wrist} \\
\cmidrule(lr){2-4}\cmidrule(lr){5-7}
& SSIM$\uparrow$ & PSNR$\uparrow$ & LPIPS$\downarrow$ &
  SSIM$\uparrow$ & PSNR$\uparrow$ & LPIPS$\downarrow$ \\
\midrule
EMA $= 0.9$ & $0.876$ & $25.81$ & $0.1105$ & $0.726$ & $20.68$ & $0.3215$ \\
No EMA  (Group Size 16, BoN 2)  & $0.876$ & $25.79$ & $0.1107$ & $0.727$ & $20.69$ & $0.3212$ \\
\bottomrule
\end{tabular}
\vspace{4pt}
\end{table}


\paragraphcustom{Learning rate.}
The choice of $\eta{=}1{\times}10^{-4}$ is the single most consequential hyperparameter: the higher rate $3{\times}10^{-4}$ degrades external PSNR by $0.50$\,dB and LPIPS by $0.007$, confirming that conservative fine-tuning is essential to avoid destabilizing the pretrained backbone (see Table~\ref{tab:abl-lr}).

\paragraphcustom{Combined reward.}
Using all three perceptual signals jointly (SSIM, PSNR, LPIPS) is the only configuration that is competitive across all six reported metrics.
Each single-metric variant wins on its own axis but regresses elsewhere — optimizing for LPIPS alone, for example, produces the best LPIPS yet the lowest SSIM among all reward configurations.
The combined reward therefore functions as a necessary regularizer, preventing the model from collapsing to a solution that sacrifices overall image quality for any single perceptual dimension (see Table~\ref{tab:abl-reward}).

\paragraphcustom{Prefix sampling strategy.}
Uniform random sampling over the full prefix range ($P\!\sim\!\mathrm{Unif}\{0,\ldots,9\}$) matches or exceeds every fixed-window baseline and, crucially, outperforms the growing curriculum introduced by~\cite{wang2026worldcompass} on five of six metrics.
This is a meaningful finding: the added complexity of a structured curriculum schedule does not translate into improved visual quality in our setting; a simple uniform draw over the full rollout-depth spectrum is both sufficient and preferable.
The marginal additional gain from restricting the range to $\{0,\ldots,4\}$ is below $0.001$ SSIM and $0.01$\,dB PSNR, confirming that exposing training to deeper rollout positions (higher $P$) does not hurt, while better covering the error regimes encountered at inference (see Table~\ref{tab:abl-curriculum}).

\paragraphcustom{Group size and best-of-$N$ filtering.}
We had found that best-of-$2$ yields further metric gains on every axis (e.g., $+0.27$\,dB external PSNR, $-0.006$ external LPIPS, $-0.018$ wrist LPIPS), suggesting that more aggressive output filtering is a straightforward avenue for future improvement — achievable without any change to the training objective or model architecture. However, for our large-scale training we stuck to a more moderate best-of-$5$ instead of the extremes to balance efficiency and prioritising stable and consistent reward signal across training batches (see Table~\ref{tab:abl-group}).

\paragraphcustom{Prediction horizon.}
$H{=}1$ consistently outperforms $H{=}3$ across both camera views.
Single-step post-training provides a tighter, lower-variance gradient signal that more directly targets the per-step generation quality evaluated at test time; multi-step rollout objectives introduce compounding errors that appear to add noise to the update without providing a commensurate benefit (see Table~\ref{tab:abl-horizon}).

\paragraphcustom{View-Specific Weighting.} 
The view weighting performed as expected -- assigning higher weight to the wrist view resulted in a post-trained model that better generates the wrist view camera. However, the practical improvement is not significant compared to case of $w_{\text{wrist}}=1$, and resulted in considerable decrease in the performance of the external cameras. Therefore, we stuck to the standard setting of $w_{\text{wrist}}=1$ (see Table~\ref{tab:abl-view}).

\paragraphcustom{KL regularization.}
Adding KL regularization improves five of six metrics, with the most consistent gains on the wrist camera.
The regularization term prevents the fine-tuned policy from drifting too far from the pretrained reference, acting as an implicit constraint that preserves the model's generalization while allowing reward-directed improvement (see Table~\ref{tab:abl-kl}).

\paragraphcustom{Old-policy EMA schedule.}
The EMA schedule of the old policy rising to $0.5$ is marginally better than a copy reference across most metrics. The EMA value rises to $0.5$ until $500$ training steps, and then remains fixed at this value.
Although the absolute differences are within measurement noise, the schedule provides a principled mechanism to prevent the on-policy sampling model from drifting too far from the original model and introducing instability early on in the training (see Table~\ref{tab:abl-ema-old}).

\paragraphcustom{EMA on the learning policy.}
Similar to~\cite{zheng2025diffusionnft}, we test with using a EMA of the weights of the policy being learned. However, we found that this does not lead to any significant improvements for the additional machinery (see Table~\ref{tab:abl-ema-learn}).  

\section{World Model For Policy Evaluation}
\label{app:wm-policy-eval}

We test whether world model can be used for policy evaluation by rolling out the learned policy in the world model and calculating the task progression rates for the different policies.

\begin{figure}[tb]
  \centering
  \begin{subfigure}{0.95\linewidth}
  \includegraphics[width=\textwidth]{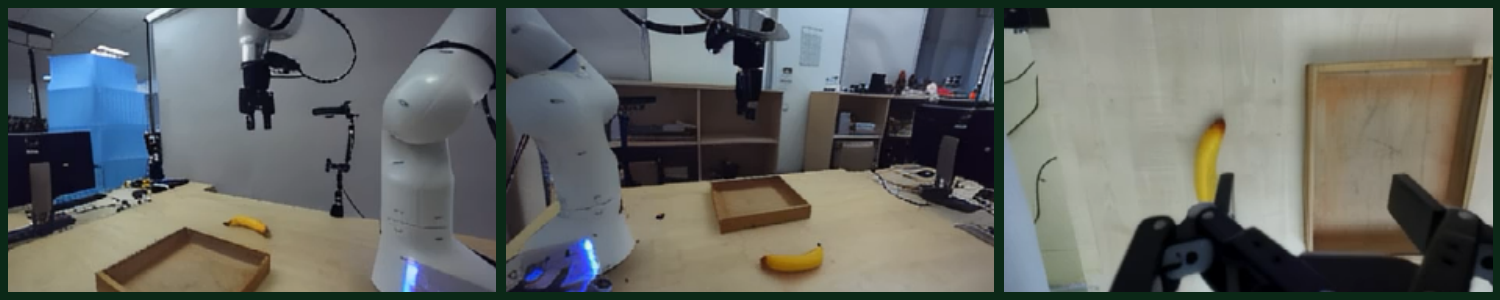}
  \caption{\textbf{Put Banana in Box}}
  \label{fig:sim2real-images-a}
  \end{subfigure}
  \hfill
  \begin{subfigure}{0.95\linewidth}
    \includegraphics[width=\textwidth]{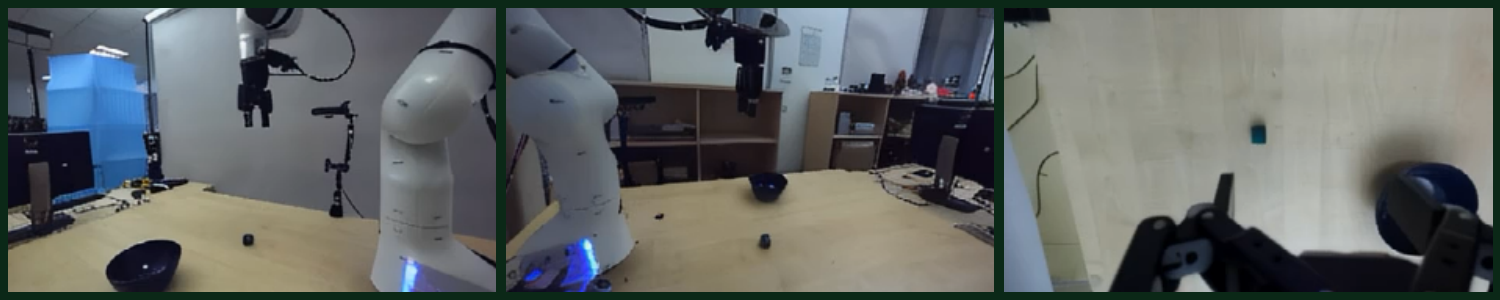}
    \caption{\textbf{Put Green Block in Bowl}}
    \label{fig:sim2real-images-b}
  \end{subfigure}
  \hfill
  \begin{subfigure}{0.95\linewidth}
    \includegraphics[width=\textwidth]{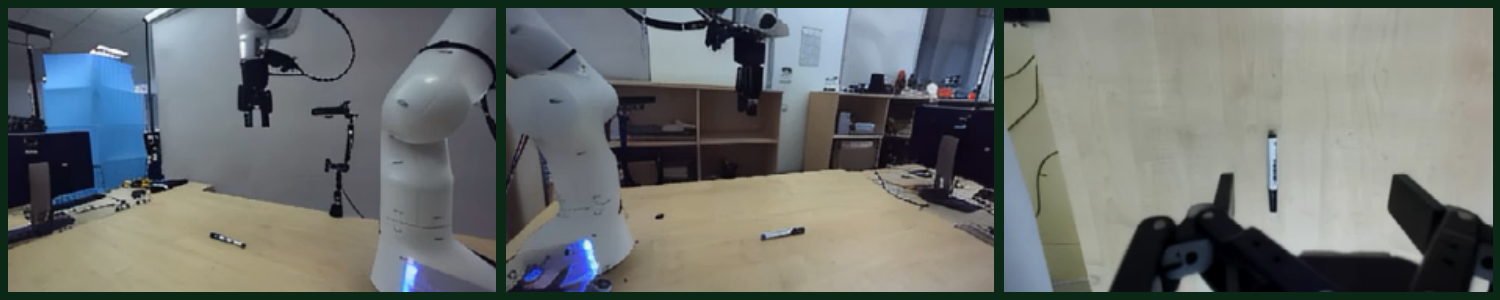}
    \caption{\textbf{Rotate Marker}}
    \label{fig:sim2real-images-c}
  \end{subfigure}
  \caption{
      \textbf{Initial Conditions for the WM-to-Real Correlation.} We run rollouts in the world models given these starting conditions for the three tasks. For the policy rollout, the middle view (external camera) and the wrist view are provided as inputs. 
  }
  \label{fig:sim2real-images}
\end{figure}

\noindent We perform a set of $3$ tasks:
\begin{itemize}
  \item \textbf{Put Banana in Box:} In this task, the robot must pick up a banana and place it inside a box. The task is successful if the banana is fully contained within the box (see Fig.~\ref{fig:sim2real-images-a}).
  \item \textbf{Put Green Block in Bowl:} In this task, the robot must pick up a green block and place it inside a bowl. The task is successful if the green block is fully contained within the bowl (see Fig.~\ref{fig:sim2real-images-b}).
  \item \textbf{Rotate Marker:} In this task, the robot must pick up a marker and rotate it by at least $45$ degrees. The task is successful if the marker is rotated by the required amount (see Fig.~\ref{fig:sim2real-images-c}).
\end{itemize}

\begin{table}[h!]
\centering
\caption{\textbf{Partial Progress of the Different Tasks}.}
\vspace{-6pt}
\label{tab:sim2real-partial-progress}
\begin{tabular}{@{}c@{\hspace{5pt}}l@{\hspace{5pt}}l@{}}
\toprule
Skill & Task & Progression\\
\midrule
\multirow{2}{*}{\emph{Put}} & Put Banana in Box & \multirow{2}{*}{Reach $\rightarrow$ Grasp $\rightarrow$ Lift $\rightarrow$ Move Close $\rightarrow$ IsInside} \\
& Put Green Block in Bowl &  \\
\midrule
\emph{Rotate} & Rotate Marker & Reach $\rightarrow$ Grasp $\rightarrow$ Rotate $45^\circ$ \\
\bottomrule
\end{tabular}
\vspace{4pt}
\end{table}

For each task, we collect real-rollouts for three policies: $\pi_0$~\cite{black2024pi_0}, $\pi_0$-FAST~\cite{pertsch2025fast} and GROOT N$1.5$~\cite{bjorck2025gr00t}. For each task and policy we collect $5$ real rollouts, and $11$ simulated rollouts in the world model. For each rollout, we record the partial progress according to Table~\ref{tab:sim2real-partial-progress}. The total progress is divided into $N$ steps and completing a particular step in order would grant $1/N$ towards the total progress. We then average the partial progress over the multiple rollouts to obtain an estimate of the policy's performance on the particular task.    

We report the Pearson Correlation $r$ and MMRV~\cite{simpler}. Higher Pearson correlation implies that the world model more closely follows the trend in real progress rates amongst the policies. MMRV, on the other hand, penalizes policy rank violations between the simulated setup (world model) and the real setup. Our world model produces higher correlation and lower MMRV than baseline~\cite{ctrlworld}.

\begin{figure}[htp]
  \centering
  \includegraphics[width=\textwidth]{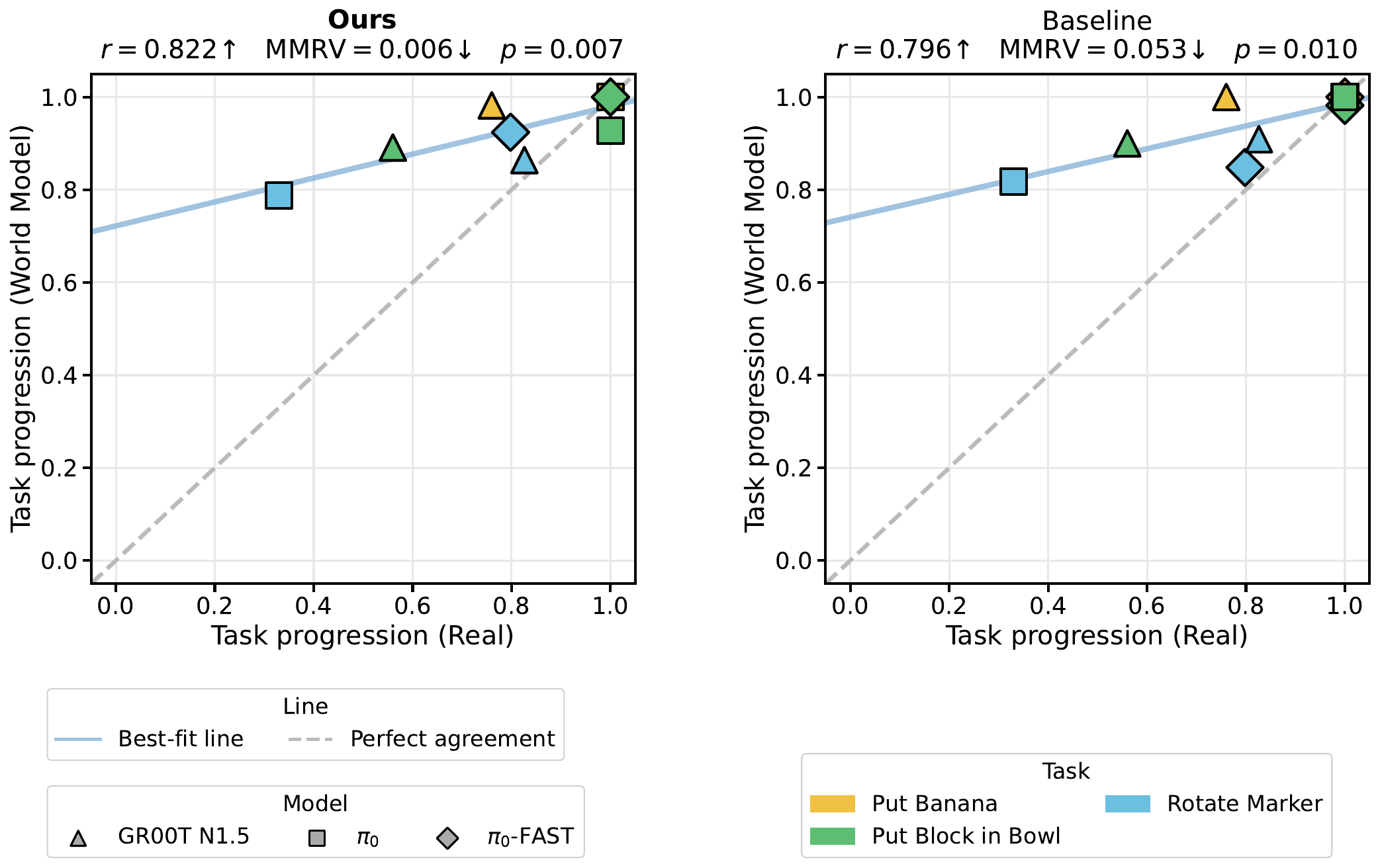}
  \caption{\small \textbf{WM to Real comparison:} While both the world models tend to make the task easier for the policies (shown by the higher values task performance on the WM), our post-trained model maintains better correlation and MMRV amongst the policies compared to the baseline~\protect\cite{ctrlworld}.}
  \label{fig:sim-real-comparison}
\end{figure}

\clearpage


\section{Human Preference Study Details}
\label{app:human-preference-study-details}

\paragraphcustom{Study design.} We conduct a human preference study to evaluate the visual quality of the rollouts generated by the original and post-trained Ctrl-World models. The study was conducted in-house on our custom-created platform. The study comprised $8$ users who are experts in computer vision, machine learning and robotics with papers in these fields. Each user also additionally had a masters qualification, with some users also having a PhD qualification. We believe that the users were sufficiently capable of discerning differences and alignment to the reference rollout video. To provide high quality signal, users were also allowed to choose which video they would like to rate. We find an inter-rater kappa $\kappa ~\sim 0.4$, which points to moderate agreement on rating. However, we find that the binomial test gives us a $p$-value of $p=3.5\cdot10^{-20}$, which means essentially zero chance that our post-trained model could have won purely by random chance. The $95\%$ CI comes out to be $[72\%,100\%]$ which is significantly higher than $50\%$.

\paragraphcustom{UI Design.}
Each user first went through an onboarding phase, where the task description and potential hints for rating the videos were provided (see, Fig.~\ref{fig:ui-onboarding}). 

On the comparison webpage, the users were shown the reference video and two generated videos (one from each checkpoint). Each video was generated using $14$ autoregressive interaction steps (approx. $11$ seconds). Only the reference video was labeled and the other videos' labels were masked. The users had the ability to scrub through the frames to view them individually or watch the entire video. Below the videos, users were provided two options ``A'' or ``B'', and were asked to choose between one of them (see, Fig.~\ref{fig:ui-comparison}). The position of the generated videos were always randomized so as to remove any position bias.

\begin{figure}[htp]
    \centering
    \includegraphics[height=0.85\textheight]{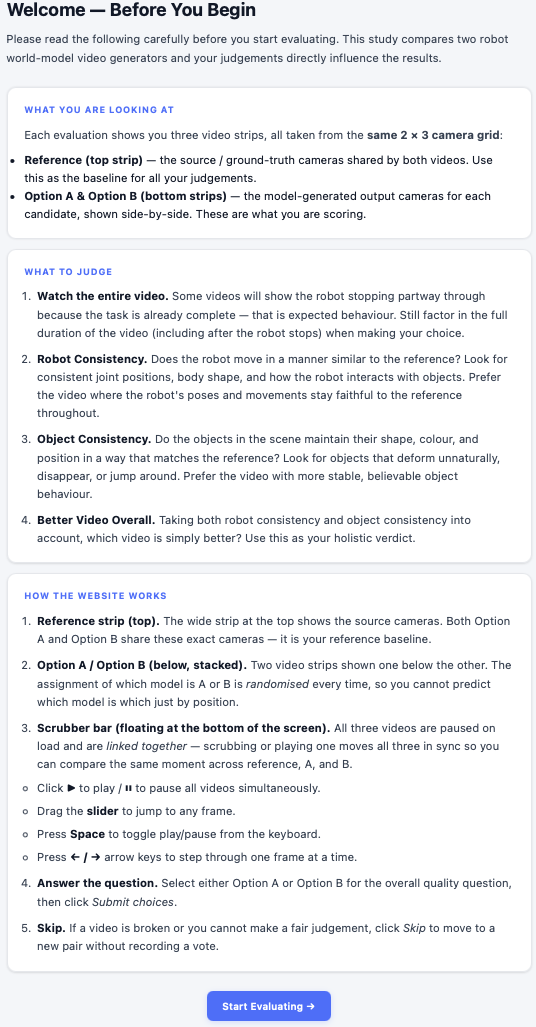}
    \caption{\small \textbf{2AFC Website Onboarding:} We provide a small description of the task and potential hints to look at when judging the quality of the generated rollouts compared to the reference.}
    \label{fig:ui-onboarding}
\end{figure}

\begin{figure}[htp]
    \centering
    \includegraphics[width=\linewidth]{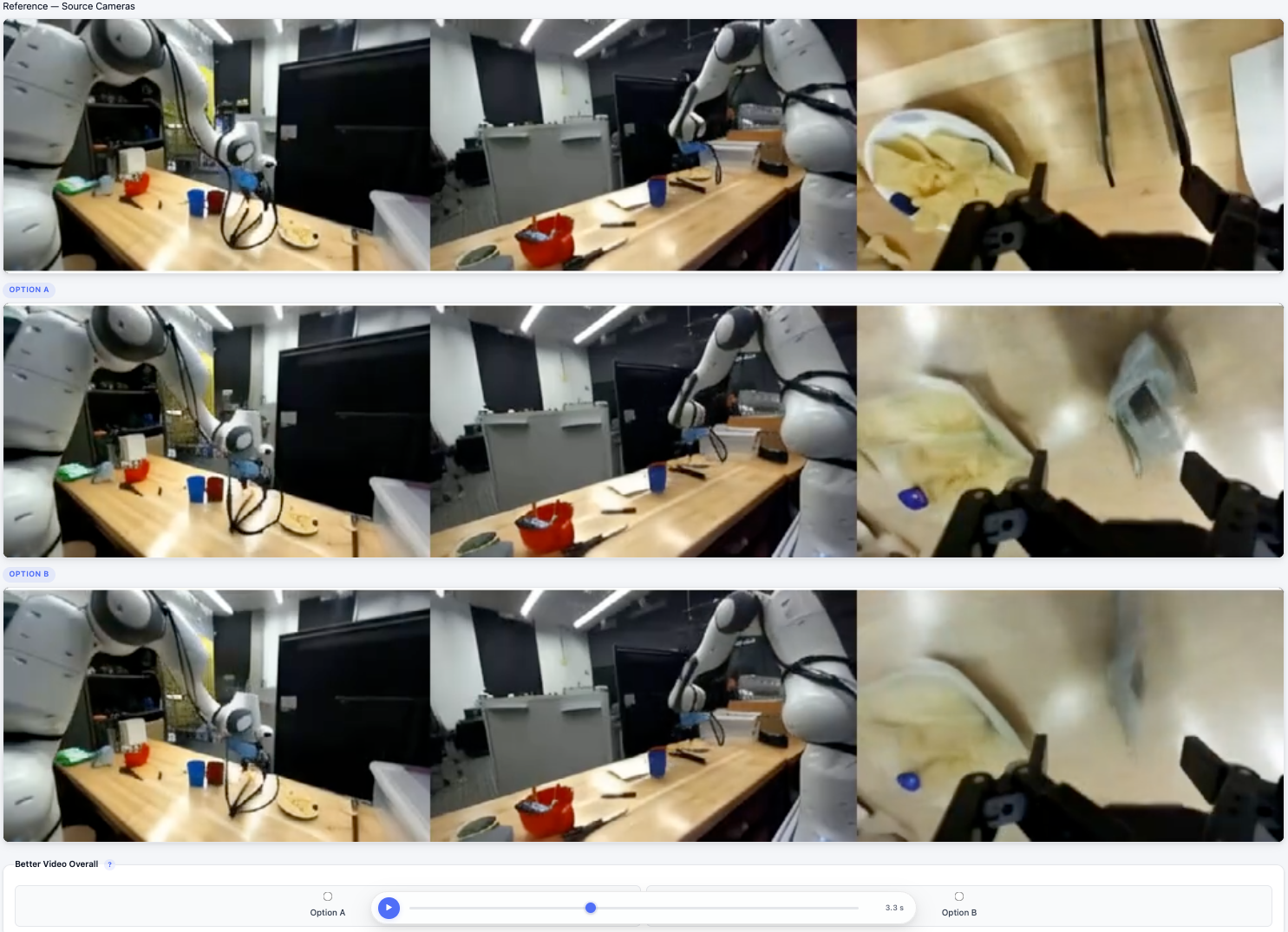}
    \caption{\small \textbf{2AFC Website Comparison UI:} On the comparison webpage, the users are provided with the reference video on the top and the two generated videos below it. The users can use the scrubber to move through the frames or watch the entire video. After which, the user needs to choose between Option A or Option B.}
    \label{fig:ui-comparison}
\end{figure}

\paragraphcustom{ELO Score Design.} 
We adopt a chess-style ELO rating system~\cite{elo1978rating} to rank models from pairwise preference votes.
Each model is initialised with a rating of $800$.
For every vote submitted by a user, the two models shown in that comparison are treated as opponents in a single ELO match.
Prior to updating their ratings, the system computes each model's \emph{expected score} via the standard logistic function,
\begin{equation}
  E_i = \frac{1}{1 + 10^{(R_j - R_i)/400}},
  \label{eq:elo-expected}
\end{equation}
where $R_i$ and $R_j$ are the current ratings of model $i$ and its opponent $j$, respectively.
$E_i \in (0,1)$ maps the rating gap to a win probability: a model rated $400$ points above its opponent is expected to win $\approx 91\%$ of the time.
The \emph{actual score} is $S_i = 1$ for the preferred model and $S_i = 0$ for the other.
Ratings are then updated symmetrically,
\begin{equation}
  R_i \;\leftarrow\; R_i + K\,(S_i - E_i),
  \label{eq:elo-update}
\end{equation}
with a fixed gain factor $K = 32$.
Because $S_i - E_i$ is small when the outcome matches the prior expectation, a model that was already the clear favourite gains little from an expected victory; conversely, an upset win produces a large positive update.
In practice, with only two models, ELO converges quickly to reflect the empirical preference rate — a model preferred 70\% of the time will accumulate a rating roughly $100$-$150$ points above the other.

\section{Additional Details}
\label{app:algorithm}

\subsection{Inference Details}
We run inference using the Euler Sampler, and run the sampling for $50$ steps. Following~\cite{zheng2025diffusionnft}, we do not use CFG while sampling the generated video. 

\subsection{Additional Results}
\label{app:additional-results}

\paragraphcustom{Degradation of the metrics over rollout length for external views}
We present in Fig.~\ref{fig:external-cam-temporal-evolution} the temporal evolution of the all the metrics for the external camera view. We find that the post-trained model has a better performance on all frames. This suggests that the post-training procedure is effective at improving the long-term consistency of the generated videos.

\begin{figure}[tb]
  \centering
  \begin{subfigure}{0.32\linewidth}
  \includegraphics[width=\textwidth]{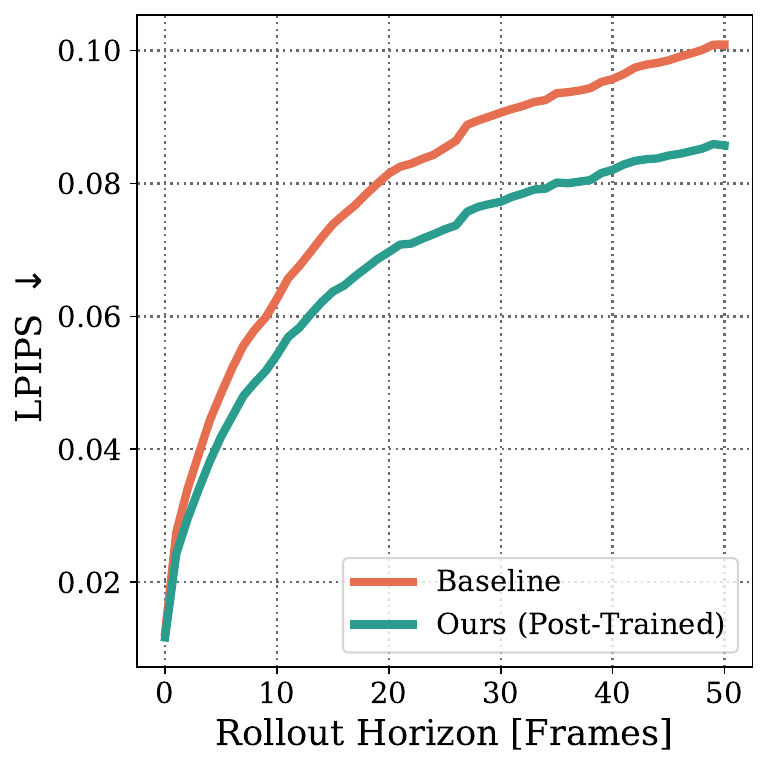}
  \end{subfigure}
  \hfill
  \begin{subfigure}{0.32\linewidth}
    \includegraphics[width=\textwidth]{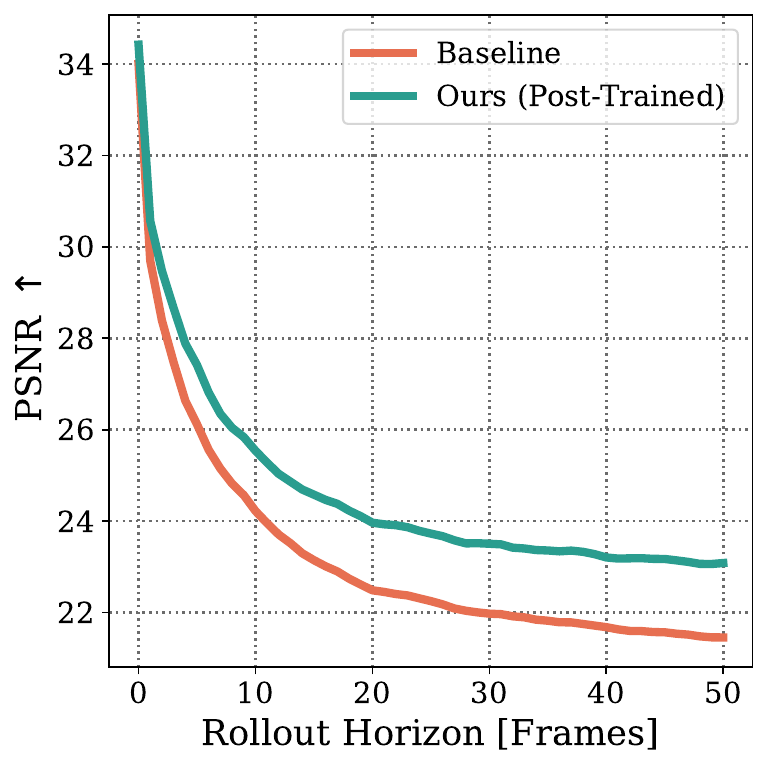}
  \end{subfigure}
  \hfill
  \begin{subfigure}{0.32\linewidth}
    \includegraphics[width=\textwidth]{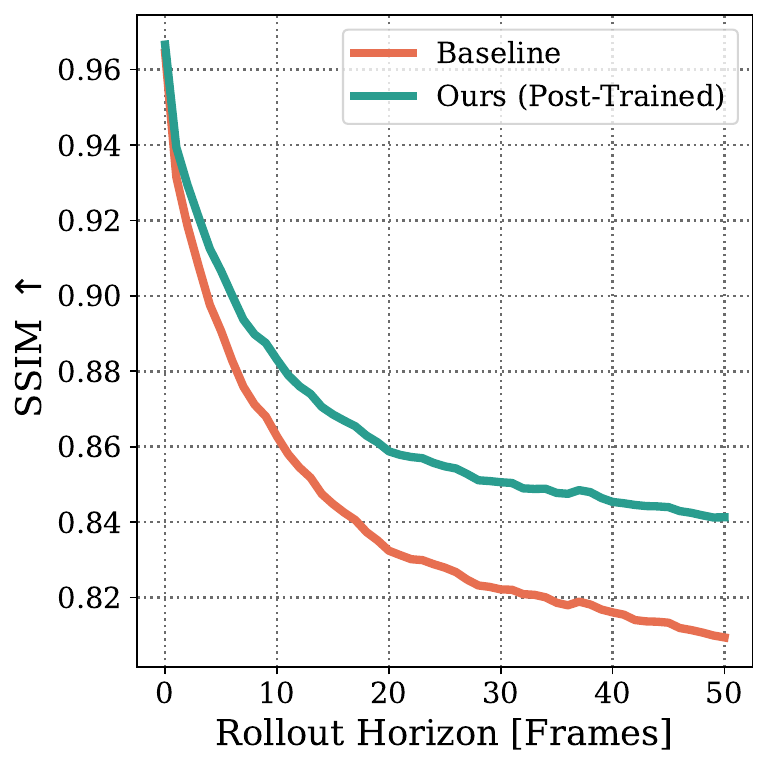}
  \end{subfigure}
  \vspace*{-2mm}
  \caption{
      \textbf{Temporal evolution of external camera metrics.} While both models exhibit natural degradation over longer horizons (x-axis), our post-trained model, PersistWorld (green), consistently maintains higher fidelity and slower error accumulation compared to the baseline (orange). Specifically, our method preserves a higher PSNR and SSIM while suppressing LPIPS drift, effectively extending the stable prediction horizon for complex, fine-grained interactions.
  }
  \label{fig:external-cam-temporal-evolution}
  \vspace*{-4mm}
\end{figure}

\paragraphcustom{Ours vs. WorldCompass~\cite{wang2026worldcompass}}
We share the contrastive-RL paradigm but differ in four ways that preclude direct comparison: $x_0$ vs v-pred parameterization (requiring our Appendix~A re-derivation), action-conditioned manipulation vs camera-pose video-game WM, DROID's larger scale, and our additional validations (human study, policy-eval correlation). To isolate the most transferable component, viz. their reward design, we replaced our reward with HPSv3~\cite{ma2025hpsv3} in our framework on DROID. Table~\ref{tab:ours-worldcompass-hps} shows that our reward structure results in better downstream performance while being much faster to train with. HPSv3 runs a large VLM to check for alignment between a text prompt and generated images. In our case we do not have such text conditioning so we instead run HPSv3 reward with the prompt: ``Robot arm <task>''. Therefore, as such it cannot assess action-conditioned trajectory consistency which is critical for our use case. Our reference-based perceptual rewards directly ground generated frames against the GT trajectory implied by the action sequence, providing a more appropriate training signal for our setting.

\begin{table}[h]
\centering
\caption{\textbf{Our reward vs HPSv3~\protect\cite{ma2025hpsv3}:} To compare to WorldCompass~\protect\cite{wang2026worldcompass}, we compare their reward structure with HPSv3~\protect\cite{ma2025hpsv3} to our light-weight rewards.}
\vspace{-6pt}
\label{tab:ours-worldcompass-hps}
\begin{tabular}{lcccc}
\toprule
Reward & SSIM $\uparrow$ & PSNR $\uparrow$ & LPIPS $\downarrow$ & training sec/iter $\downarrow$ \\
\midrule
HPSv3 (theirs) & 0.802 & 22.58 & 0.193 & 129.2 \\
\textbf{Ours} (top-$k{=}2$) & \textbf{0.826} & \textbf{24.09} & \textbf{0.181} & \textbf{41.3}  \\
\bottomrule
\end{tabular}
\vspace{-2pt}
\end{table}

\paragraphcustom{Generalization beyond DROID}
We finetuned the SVD backbone on Bridge (WidowX)~\cite{ebert2021bridge} and RT-1 (Google Robot)~\cite{rt1} for 100k steps, then applied PersistWorld post-training. In table~\ref{tab:extra-results-gen-droid}, we compare against IRASim~\cite{zhu2025irasim} ($^\dagger$IRASim Tab.~2) under matched data and resolution. PersistWorld improves all metrics on both embodiments, with gains larger than on DROID.

\begin{figure}[tb]
  \centering
  \includegraphics[width=0.48\textwidth]{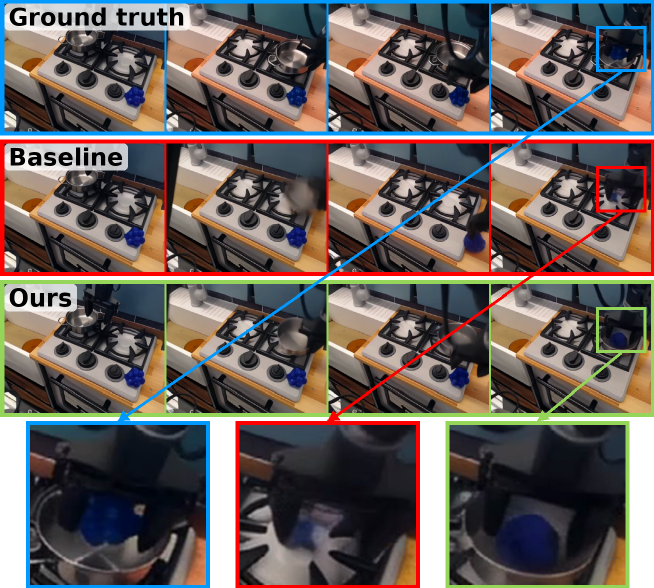}
  \includegraphics[width=0.48\textwidth]{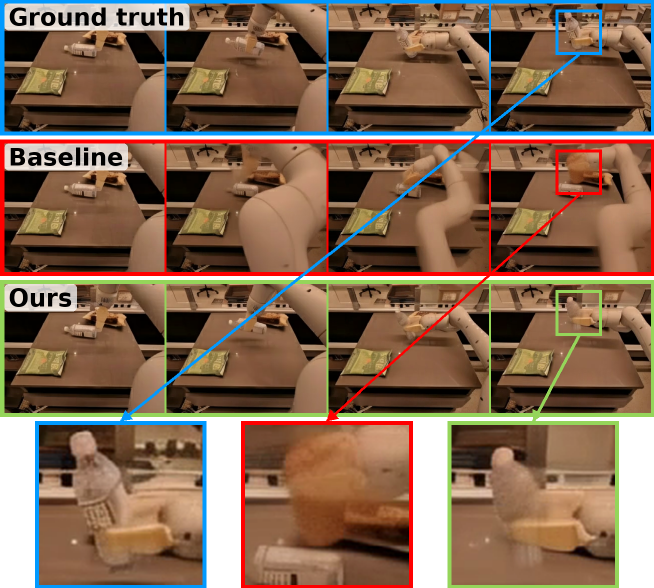}
  \caption{
  \textbf{Qualitative comparison of autoregressive rollout stability on Bridge~\cite{ebert2021bridge} and RT-1~\cite{rt1}.} We compare long-horizon (11s) generations from the \textcolor{myRed}{baseline} against our \textcolor{myGreen}{PersistWorld}. In both the cases, we see that the \textcolor{myRed}{baseline} model struggles to maintain object consistency. \textbf{In left,} we see that the saucepan has completely disappeared in the baseline model, while PersistWorld is able to maintain it. \textbf{In right,} we see that the baseline model has lost the bottle, and instead hallucinates a different object within the robot's gripper. In contrast, PersistWorld is able to maintain the bottle and correctly shows the gripper grasping the bottle.
}
    \vspace{-5mm}
  \label{fig:other-embodiment-qres}
\end{figure}

\begin{table}[h]
\centering
\caption{\textbf{Generalization Beyond DROID:} We compare the performance of the standard action conditioned pretraining on SVD backbone to PersistWorld on a long rollout with $10$ consecutive interactions. We also report IRASim~\protect\cite{zhu2025irasim} on their long trajectory setup ($\dagger$ indicates the number is taken from original paper).}
\vspace{-6pt}
\label{tab:extra-results-gen-droid}
\begin{tabular}{llccc}
\toprule
Dataset & Method & SSIM $\uparrow$ & PSNR $\uparrow$ & LPIPS $\downarrow$ \\
\midrule
\multirow{3}{*}{Bridge} & IRASim~\cite{zhu2025irasim} (long-traj)$^\dagger$ & -- & 23.26 & -- \\
 & SVD backbone & 0.864 & 25.20 & 0.1061 \\
 & + PersistWorld (\textbf{Ours}) & \textbf{0.876} & \textbf{26.36} & \textbf{0.0957} \\
\midrule
\multirow{3}{*}{RT-1} & IRASim~\cite{zhu2025irasim} (long-traj)$^\dagger$ & -- & 24.62 & -- \\
 & SVD backbone & 0.833 & 23.81 & 0.1683 \\
 & + PersistWorld (\textbf{Ours}) & \textbf{0.865} & \textbf{26.63} & \textbf{0.1305} \\
\bottomrule
\end{tabular}
\end{table}

In Figure~\ref{fig:other-embodiment-qres}, we visualize the improvements by PersistWorld on the Bridge and RT-1 datasets. We find that the post-trained model is able to generate more realistic and consistent rollouts compared to the baseline SVD backbone.

\paragraphcustom{Comparison to Self-Forcing style baseline.}
Recent works such as Self-Forcing~\cite{self-forcing, rolling-forcing} are not directly applicable to our case since they primarily distill from a (usually larger) bidirectional teacher, which is not available in our case. However, to measure the importance of our RL objective over the training on corrupted history, we post-trained the model on its own rollout history but ask it to regress the ground truth latents. On the validation set with rollout length of 10 interactions, the visual metrics favor ours: wrist \textbf{0.721}/\textbf{20.40}/0.338 \vs SF 0.710/20.17/0.338; external \textbf{Ours }\textbf{0.872}/\textbf{25.52}/\textbf{0.116} \vs 0.868/25.35/0.119. The advantage is very clear when we look at the effective reward obtained on the rollouts. Figure.~\ref{fig:ours-sf-reward} shows that our post-training strategy obtains an effective reward at step 240 that exceeds the baseline's reward at 2400 steps.

\begin{figure}
    \centering
    \includegraphics[width=0.5\linewidth]{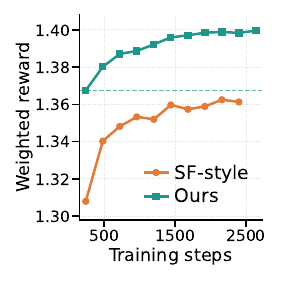}
    \caption{\textbf{Ours vs SF style baseline:} Our reward exceeds the SF-style baseline. Our procedure achieves the same performance $\sim10\times$ faster and achieves a much higher reward at the same number of steps.}
    \label{fig:ours-sf-reward}
    \vspace{-6mm}
\end{figure}

\begin{figure}
    \centering
    \includegraphics[width=0.9\linewidth]{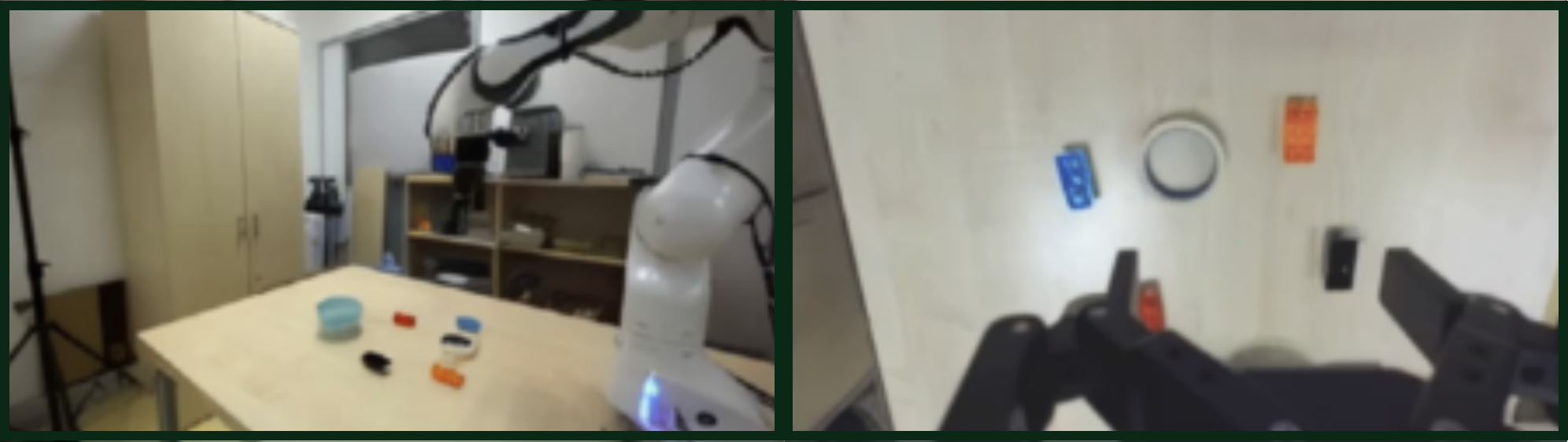}
    \caption{\textbf{Policy Improvement Setup:} We place 4 blocks on a table and command the robot to pick up the top-right/top-left/bottom-right/bottom-left block in the $2\times2$ grid. We also place some distractors on the table such as the bowl and the tape.}
    \label{fig:policy-improvement-setup}
    \vspace{-6mm}
\end{figure}

\paragraphcustom{Policy Improvement.}
We ran a Ctrl-World-style policy improvement experiment: $\pi_0$~\cite{black2024pi_0} is full-finetuned on 50 rollouts for 4 block spatial picking task from each WM, then evaluated on a real-robot (5 trials/block). Prompt: ``Pick up the top-left/top-right/bottom-left/bottom-right block''.
Average task progression: base 0.350, +Ctrl-World 0.558, +\textbf{Ours 0.667} (\textbf{$\mathbf{+19\%}$ over Ctrl-World}). Figure.~\ref{fig:policy-improvement-setup} shows the setup for the policy improvement experiment.
\vspace{-1mm}

\subsection{Post-Training Algorithm}
We present the Post-Training algorithm in Alg.~\ref{alg:rl-posttraining}. We describe the four steps $S_1$ to $S_4$. 


\begin{algorithm}[t]
\caption{RL Post-Training for Autoregressive Video World Models}
\label{alg:rl-posttraining}
\begin{algorithmic}[1]
\Require Pretrained world model $\theta$; frozen reference policy $\hat{\mathbf{x}}_0^{\mathrm{old}}$ (EMA copy of $\theta$); dataset $\mathcal{D}$; group size $K{=}16$; branch length $F$; mixing coefficient $\beta$; reward weights $w_{\mathrm{LPIPS}},\,w_{\mathrm{SSIM}},\,w_{\mathrm{PSNR}}$
\For{each training step}
  \State \textbf{\textit{-- Stage $\text{S}_1$: Generate shared prefix --}}
  \State Sample initial ground-truth observation $(\mathbf{x}_0, \mathbf{e}_0)$ from $\mathcal{D}$
  \State Sample prefix length $P \sim \mathrm{Unif}\{0, 1, \ldots, 9\}$
  \State Initialise history buffer $\mathbf{h}_0$ by replicating the encoded latent of $\mathbf{x}_0$
  \For{$p = 1$ \textbf{to} $P$} \Comment{closed-loop autoregressive rollout}
    \State Sample $\mathbf{x}_p \sim \theta(\cdot\,|\,\mathbf{h}_{p-1},\, \mathbf{e}_{p-1:p+L})$
    \State $\mathbf{h}_p \leftarrow$ append $\mathrm{Enc}(\mathbf{x}_p)$ to $\mathbf{h}_{p-1}$ \Comment{history buffer update}
  \EndFor
  \Statex
  \State \textbf{\textit{-- Stage $\text{S}_2$: Branch $K$ candidate continuations --}}
  \For{$k = 1$ \textbf{to} $K$} \Comment{independent samples from shared context}
    \State Initialise private buffer $\mathbf{h}^{(k)} \leftarrow \mathbf{h}_P$ \Comment{frozen copy of prefix}
    \For{$f = 1$ \textbf{to} $F$}
      \State Sample $\mathbf{x}^{(k)}_{P+f} \sim \theta\bigl(\cdot\,|\,\mathbf{h}^{(k)},\, \mathbf{e}_{P+f:P+f+L}\bigr)$
      \State $\mathbf{h}^{(k)} \leftarrow$ append $\mathrm{Enc}\bigl(\mathbf{x}^{(k)}_{P+f}\bigr)$ to $\mathbf{h}^{(k)}$
    \EndFor
  \EndFor
  \Statex
  \State \textbf{\textit{-- Stage $\text{S}_3$: Score, rank, and group-normalise --}}
  \For{$k = 1$ \textbf{to} $K$}
    \State Compute per-view, per-frame LPIPS, SSIM, PSNR against ground-truth frames
    \State $R^{(k)} \leftarrow -w_{\mathrm{LPIPS}}\,\overline{\mathrm{LPIPS}}^{(k)} + w_{\mathrm{SSIM}}\,\overline{\mathrm{SSIM}}^{(k)} + w_{\mathrm{PSNR}}\,\overline{\mathrm{PSNR}}^{(k)}$
  \EndFor
  \State $\mu_R \leftarrow \frac{1}{K}\sum_k R^{(k)},\quad \sigma_R \leftarrow \mathrm{std}_k(R^{(k)})$
  \For{$k = 1$ \textbf{to} $K$}
    \State $A^{(k)} \leftarrow \bigl(R^{(k)} - \mu_R\bigr)\,/\,(\sigma_R + \epsilon)$ \Comment{z-score normalisation}
    \State $r^{(k)} \leftarrow \bigl(\mathrm{clip}(A^{(k)},\,-1,\,1) + 1\bigr)\,/\,2$ \Comment{rescale to $[0,1]$}
  \EndFor
  \Statex
  \State \textbf{\textit{-- Stage $\text{S}_4$: Contrastive denoising update --}}
  \State $\mathcal{L} \leftarrow 0$
  \For{$k = 1$ \textbf{to} $K$}
    \State Sample $\sigma \sim p(\sigma)$,\; $\boldsymbol{\varepsilon} \sim \mathcal{N}(\mathbf{0}, \mathbf{I})$
    \State $\mathbf{x}^{(k)}_\sigma \leftarrow \mathbf{x}^{(k)}_0 + \sigma\boldsymbol{\varepsilon}$ \Comment{forward noising}
    \State $\hat{\mathbf{x}}^{(k)}_{0,\theta} \leftarrow$ current model prediction from $\mathbf{x}^{(k)}_\sigma$
    \State $\hat{\mathbf{x}}^{(k)}_{0,\mathrm{old}} \leftarrow$ reference model prediction from $\mathbf{x}^{(k)}_\sigma$ \Comment{frozen}
    \State $\hat{\mathbf{x}}^{(k)+}_{0} \leftarrow (1-\beta)\,\hat{\mathbf{x}}^{(k)}_{0,\mathrm{old}} + \beta\,\hat{\mathbf{x}}^{(k)}_{0,\theta}$
    \State $\hat{\mathbf{x}}^{(k)-}_{0} \leftarrow (1+\beta)\,\hat{\mathbf{x}}^{(k)}_{0,\mathrm{old}} - \beta\,\hat{\mathbf{x}}^{(k)}_{0,\theta}$
    \State $\mathcal{L} \mathrel{+}= r^{(k)}\,\|\hat{\mathbf{x}}^{(k)+}_{0} - \mathbf{x}^{(k)}_0\|_2^2 + (1-r^{(k)})\,\|\hat{\mathbf{x}}^{(k)-}_{0} - \mathbf{x}^{(k)}_0\|_2^2$
  \EndFor
  \State Update $\theta$ (LoRA adapters + action encoder only) via gradient descent on $\mathcal{L}$
  \State $\hat{\mathbf{x}}_0^{\mathrm{old}} \leftarrow \mathrm{EMA}(\hat{\mathbf{x}}_0^{\mathrm{old}},\, \theta)$ \Comment{update reference policy}
\EndFor
\end{algorithmic}
\end{algorithm}


%
%
\end{document}